\documentclass{article}
\usepackage[utf8]{inputenc} %
\usepackage[T1]{fontenc}    %

\usepackage[top=1in, bottom=1in, left=1in, right=1in]{geometry}
\usepackage{parskip}

\usepackage{microtype}
\usepackage{graphicx}
\usepackage{subfigure}
\usepackage{booktabs} %

\usepackage{amsmath}
\usepackage{amssymb}
\usepackage{mathtools}
\usepackage{amsthm}
\usepackage{bbm}
\usepackage{algorithm}
\usepackage{algpseudocode}

\usepackage{hyperref}
\usepackage{natbib}

\usepackage[capitalize,noabbrev]{cleveref}

\theoremstyle{plain}
\newtheorem{theorem}{Theorem}[section]

\newtheorem{lemma}[theorem]{Lemma}

\theoremstyle{definition}
\newtheorem{definition}[theorem]{Definition}
\newtheorem{assumption}[theorem]{Assumption}
\theoremstyle{remark}
\newtheorem{remark}[theorem]{Remark}

\usepackage[textsize=tiny]{todonotes}

\newcommand{\E}{\mathbb{E}}
\DeclareMathOperator*{\argmin}{arg\,min}
\newcommand{\alg}{\mathcal{A}}
\newcommand{\WW}{\mathcal{W}}
\newcommand{\XX}{\mathcal{X}}
\newcommand{\norm}[1]{\left\lVert#1\right\rVert}
\newcommand{\wt}{\widetilde}
\newcommand{\bx}{\mathbf{X}}

\newcommand\blfootnote[1]{%
    \begingroup
    \renewcommand\thefootnote{}\footnote{#1}%
    \addtocounter{footnote}{-1}%
    \endgroup
}

\title{
Private Heterogeneous Federated Learning Without a Trusted Server Revisited: Error-Optimal and Communication-Efficient Algorithms for Convex Losses
}
\author{
Changyu Gao\(^{*,1}\), 
Andrew Lowy\(^{*,1}\), 
Xingyu Zhou\(^{*,2}\), 
Stephen J. Wright\(^{1}\)\\
\\
\(^{1}\)University of Wisconsin-Madison, 
\(^{2}\)Wayne State University \\
\blfootnote{The first three authors are ordered alphabetically. Correspondence to Andrew
Lowy <alowy@wisc.edu>, Changyu Gao <changyu.gao@wisc.edu>}}%

\begin{document}
\maketitle

\begin{abstract}
We revisit the problem of federated learning (FL) with private data from people who do not trust the server or other silos/clients. In this context, every silo (e.g.\ hospital) has data from several people (e.g.\ patients) and needs to protect the privacy of each person's data (e.g.\ health records), even if the server and/or other silos try to uncover this data. Inter-Silo Record-Level Differential Privacy (ISRL-DP) prevents each silo's data from being leaked, by requiring that silo \(i\)'s \textit{communications} satisfy item-level differential privacy. Prior work~\citep{lowy2022private} characterized the optimal excess risk bounds for ISRL-DP algorithms with \textit{homogeneous} (i.i.d.) silo data and convex loss functions. However, two important questions were left open: (1) Can the same excess risk bounds be achieved with \textit{heterogeneous} (non-i.i.d.) silo data? (2) Can the optimal risk bounds be achieved with \textit{fewer communication rounds}? 
In this paper, we give positive answers to both questions. 
We provide novel ISRL-DP FL algorithms that achieve the optimal excess risk bounds in the presence of heterogeneous silo data. Moreover, our algorithms are more \textit{communication-efficient} than the prior state-of-the-art. For smooth loss functions, our algorithm achieves the \textit{optimal} excess risk bound 
and has communication complexity that matches the non-private lower bound.
Additionally, our algorithms are more \textit{computationally efficient} than the previous state-of-the-art. 
\end{abstract}

\section{Introduction}
Federated learning (FL) is a distributed machine learning paradigm in which multiple \textit{silos} (a.k.a.\ clients), such as hospitals or cell-phone users, collaborate to train a global model. In FL, silos store their data locally and exchange focused updates (e.g.\ stochastic gradients), sometimes making use of a central server~\citep{kairouz2021advances}. FL has been applied in myriad domains, from consumer digital products (e.g.\ Google's mobile keyboard~\citep{hard2018federated} and Apple's iOS~\citep{apple2}) to healthcare~\citep{courtiol2019deep}, finance~\citep{fedai19}, and large language models (LLMs)~\citep{hilmkil2021scaling}.

One of the primary reasons for the introduction of FL was to enhance protection of the privacy of people's data~\citep{mcmahan2017originalFL}. Unfortunately, local storage is not sufficient to prevent data from being leaked, because the model parameters and updates  communicated between the silos and the central server can reveal sensitive information \citep{zhu2020deep,gupta2022recovering}. For example, \citet{gupta2022recovering} attacked a FL model for training LLMs to uncover private text data.  

\textit{Differential privacy} (DP)~\citep{dwork2006calibrating} guarantees that private data cannot be leaked. Different variations of DP have been considered for FL.\ \textit{Central DP} prevents the \textit{final trained FL model} from leaking data to an \textit{external adversary}~\citep{jayaraman2018distributed, noble2022differentially}. However, central DP has two major drawbacks: (1) it does not provide a privacy guarantee for each individual silo; and (2) it does not ensure privacy when an attacker/eavesdropper has access to the server or to another silo. 

Another notion of DP for FL is \textit{user-level DP}~\citep{mcmahan17, geyer17, levy2021learning}. User-level DP mitigates the drawback (1) of central DP, by providing privacy for the \textit{complete local data set} of each silo. User-level DP is practical in the \textit{cross-device FL} setting, where each silo is a single person (e.g.\ cell-phone user) with a large number of records (e.g.\ text messages).  However, user-level DP still permits privacy breaches if an adversary has access to the server or eavesdrops on the communications between silos. Moreover, user-level DP is not well-suited for \textit{cross-silo} federated learning, where silos represent organizations like hospitals, banks, or schools that house data from many individuals (e.g.\ patients, customers, or students). In the cross-silo FL context, each person possesses a record, referred to as an ``item,'' which may include sensitive data. Therefore, an appropriate notion of DP for cross-silo FL should safeguard the privacy of each individual record (i.e.\ ``item-level differential privacy'') within silo \(i\), rather than the complete aggregated data of silo \(i\).

Following~\citet{lowy2022private,lowy2023private,virginia} (among others), this work considers \textit{inter-silo record-level DP} (ISRL-DP). ISRL-DP requires that the full transcript of \textit{messages} sent by silo \(i\) satisfy item-level DP, for all silos \(i\). Thus, ISRL-DP guarantees the privacy of each silo's local data, even in the presence of an adversary with access to the server, other silos, or the communication channels between silos and the server. See~\cref{fig: diagram} for an illustration. 
If each silo only has one record, then ISRL-DP reduces to \textit{local DP}~\citep{whatcanwelearnprivately,duchi13}. However, in the FL setting where each silo has many records, ISRL-DP FL is more appropriate and permits much higher accuracy than local DP~\citep{lowy2023private}. Moreover, ISRL-DP implies user-level DP if the ISRL-DP parameters are sufficiently small~\citep{lowy2023private}. Thus, ISRL-DP is a practical privacy notion for cross-silo and cross-device FL when individuals do not trust the server or other silos/clients/devices with their sensitive data.

\begin{figure}
  \centering
  \includegraphics[width = 
  0.5 \textwidth]{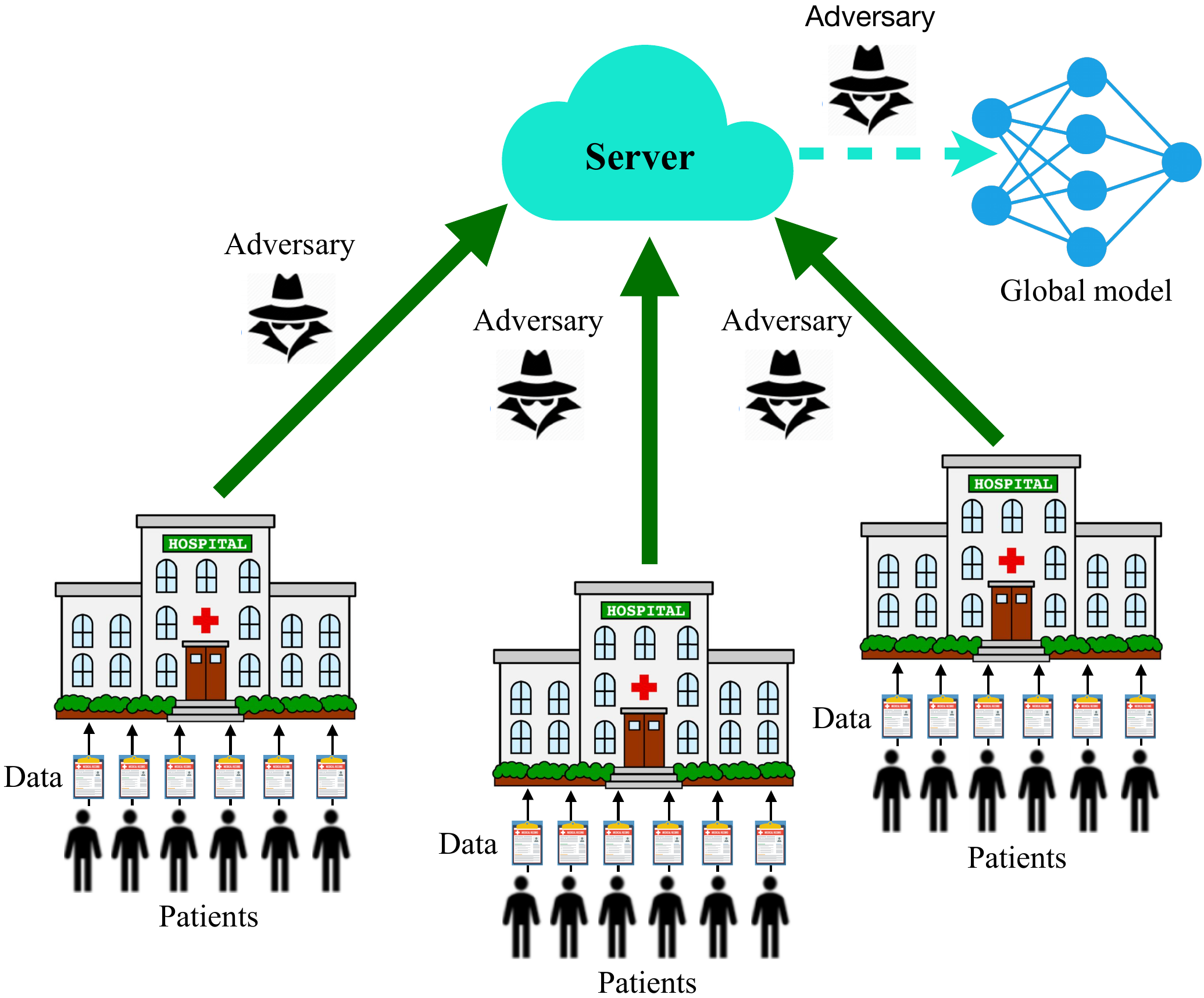}
  \caption{ \footnotesize
ISRL-DP maintains the privacy of each patient's record, provided the patient's \textit{own hospital} is trusted. Silo \(i\)'s messages are item-level DP, preventing data leakage, even if the server/other silos collude to decode the data of silo~\(i\).
}\label{fig: diagram}
\end{figure}

\paragraph{Problem Setup.}
Consider a FL environment with \(N\) silos, where each silo has a local data set with \(n\) samples: \(X_i = (x_{i, 1}, \cdots , x_{i, n})\) for \(i \in [N] := \{1,\ldots,N\}\). At the beginning of every communication round \(r\), silos receive the global model \(w_r\). Silos then utilize their local data to enhance the model, and send local updates to the server (or to other silos for peer-to-peer FL). The server (or other silos) uses the local silo updates to update the global model.

For each silo~\(i\), let \(\mathcal{D}_i\) be an unknown probability distribution on a %
data domain \(\XX\). Let \(f: \mathcal{W} \times \mathcal{X} \to \mathbb{R}\) be a loss function (e.g.\ cross-entropy loss for logistic regression), where \(\WW \subset \mathbb{R}^d\) is a parameter domain. Assume \(f(\cdot, x)\) is convex. Silo \(i\)'s local objective is to minimize its expected population/test loss, which is
\begin{equation} \label{eq: SCO}
F_i(w):= \mathbb{E}_{x_i \sim \mathcal{D}_i}[f(w, x_i)].
\end{equation} 

Our (global) objective is to 
find a model that achieves small error for all silos, by solving the FL problem
\begin{equation}\label{eq:FL}
\tag{FL}
\min_{w \in \mathcal{W}} \left\{F(w):= \frac{1}{N}\sum_{i=1}^N  F_i(w)\right\},
\end{equation}
while maintaining the privacy of each silo's local data under ISRL-DP.\@

Assume that the samples \(\{x_{i,j}\}_{i \in [N], j \in [n]}\) are independent. 
The FL problem is \textit{homogeneous} if the data is i.i.d., i.e., \(\mathcal{D}_i = \mathcal{D}_j\) for all \(i, j \in [N]\). 
In this work, we focus on the more challenging  \textit{heterogeneous} (non-i.i.d.) case: i.e.\ the data distributions \(\{\mathcal{D}_i\}_{i=1}^N\) may be arbitrary. This case arises in many practical FL settings~\citep{kairouz2021advances}. 

\paragraph{Contributions.}
The quality of a private algorithm \(\alg\) for solving Problem~\ref{eq:FL} is measured by its \textit{excess risk} (a.k.a. excess population/test loss), \(\mathbb{E} [F(\mathcal{A}(\mathbf{X}))] - F^*\), where \(F^* = \inf_{w \in \mathcal{W}} F(w)\) and the expectation is over the random draw of the data \(\mathbf{X} = (X_1, \ldots, X_N)\) as well as the randomness of \(\mathcal{A}\).\ \citet{lowy2023private} characterized the minimax optimal excess risk (up to logarithms) for convex ISRL-DP FL in the \textit{homogeneous} (i.i.d.) case: \begin{equation}\label{eq: optimal risk informal}
\wt{\Theta}\left(\frac{1}{\sqrt{N}}\left(\frac{1}{\sqrt{n}} + \frac{\sqrt{d \log(1/\delta)}}{\varepsilon n}\right)\right),
\end{equation}
where \(d\) is the dimension of the parameter space and \(\varepsilon, \delta\) are the privacy parameters. For \textit{heterogeneous} FL, the state-of-the-art excess risk bound is \(O\left(1/\sqrt{Nn} + \left(\sqrt{d\log(1/\delta)}/(\varepsilon n \sqrt{N})\right)^{4/5}\right)\), assuming smoothness of the loss function~\citep{lowy2023private}. 
\citet{lowy2023private} left the following question open: 
\begin{center}
\noindent\fbox{
    \parbox{0.8\linewidth}{
\textbf{Question 1.} Can the optimal ISRL-DP excess risk in~\eqref{eq: optimal risk informal} for solving Problem~\ref{eq:FL} be attained in the presence of \textit{heterogeneous} silo data? 
}}
    \end{center}
\textbf{Contribution 1.} We give a positive answer to \textbf{Question 1}. %
Moreover, we do not require smoothness of the loss function to attain the optimal rate: see~\cref{thm:smoothing}. 

In practical FL settings (especially cross-device settings), it can be expensive or slow for silos to communicate with the server or with each other~\citep{kairouz2021advances}. Thus, another important desiderata for FL algorithms is communication efficiency, which is measured by \textit{communication complexity}: the number of communication rounds \(R\) that are required to achieve excess risk \(\alpha\). In the \textit{homogeneous} setting, the state-of-the-art communication complexity for an algorithm achieving optimal excess risk is due to~\citet[Theorem D.1]{lowy2023private}: \begin{equation}\label{eq: LR smooth communication}
    R_{\text{SOTA}} = 
 \widetilde O\left(\min\left\{Nn, \frac{Nn^2 \varepsilon^2}{d} \right\}\right).
\end{equation}

\begin{center}
\noindent\fbox{
    \parbox{0.8\linewidth}{
\textbf{Question 2.} Can the optimal ISRL-DP excess risk in~\eqref{eq: optimal risk informal} for solving Problem~\ref{eq:FL} be attained in \textit{fewer communication rounds} \(R\), for \(R \ll R_{SOTA}\)? 
}}
    \end{center}
\textbf{Contribution 2.} We answer \textbf{Question 2} positively.
For smooth losses, our algorithm achieves optimal excess risk with significantly improved communication complexity:
\begin{equation}\label{eq: our R smooth}
R_{\text{smooth}} = \widetilde O\left(\min\left\{(Nn)^{1/4}, \frac{N^{1/4} n^{1/2} \varepsilon^{1/2}}{d^{1/4}}\right\}\right).
\end{equation}

Our communication complexity in~\eqref{eq: our R smooth} \textit{matches the non-private lower bound} of~\citet{woodworth2020minibatch} in the high heterogeneity regime (\cref{thm:lower-com}), hinting at the communication-optimality of our algorithm. 

For nonsmooth loss functions, our communication complexity is \begin{equation}\label{eq: our R nonsmooth}
R_{\text{nonsmooth}} = \widetilde O\left(\min\left\{(Nn)^{1/2}, \frac{N^{1/2} n \varepsilon}{d^{1/2}}\right\}\right),
\end{equation}
a major improvement over the prior state-of-the-art bound in~\eqref{eq: LR smooth communication}. 

Moreover, we achieve these improved communication complexity and optimal excess risk bounds \textit{without assuming homogeneity} of silo data.

When \(N=1\) in~\eqref{eq:FL}, ISRL-DP FL reduces to central DP stochastic convex optimization (SCO), which has been studied extensively~\citep{bassily2019private,feldmanPrivateStochasticConvex2020,asi2021private,zhangBringYourOwn2022}. In this centralized setting, communication complexity is usually referred to as \textit{iteration complexity}. Even in the special case of \(N=1\), \textit{our iteration complexity for smooth losses improves over the prior state-of-the-art} result~\citep{zhangBringYourOwn2022}.

Another important property of FL algorithms is \textit{computational efficiency}, which we measure by \textit{(sub)gradient complexity}: the number of stochastic (sub)gradients \(T\) that an algorithm must compute to achieve excess risk \(\alpha\). 

In the \textit{homogeneous} case, the state-of-the-art gradient complexity for an  ISRL-DP FL algorithm that attains optimal excess risk is due to~\citet[Theorem D.1]{lowy2023private}. For smooth losses and \(\varepsilon = \Theta(1)\), \citet{lowy2023private} obtained the follow gradient complexity 
\begin{equation}\label{eq: LR23 smooth gradient complexity}
 T_{\text{SOTA}} \!=\! \widetilde O\left(N^2\min\left\{n, \frac{n^2}{d}\right\} + N^{3/2}\min\left\{n^{3/2}, \frac{n^2}{d^{1/2}}\right\} \right).
\end{equation}
\begin{center}
\noindent\fbox{
    \parbox{0.8\linewidth}{
\textbf{Question 3.} Can the optimal ISRL-DP excess risk in~\eqref{eq: optimal risk informal} for solving Problem~\ref{eq:FL} be attained with \textit{smaller gradient complexity} \(T \ll T_{\text{SOTA}}\)? 
}}
    \end{center}

\textbf{Contribution 3.} We give a positive answer to \textbf{Question 3}. We provide an ISRL-DP FL algorithm with optimal excess risk and improved gradient complexity (see~\cref{thm:main}). When \(d = \Theta(n)\) and \(\varepsilon = \Theta(1)\), our gradient complexity bound simplifies to 
\begin{equation}\label{eq: our smooth gradient complexity}
    T_{\text{smooth}} = \widetilde O\left(N^{5/4}n^{1/4} + (Nn)^{9/8} \right). 
\end{equation}
In~\cref{thm:nonsmooth}, we also improve over the state-of-the-art subgradient complexity bound of~\citet[Theorem D.2]{lowy2023private} for \textit{nonsmooth} losses. 
Moreover, in contrast to these earlier results, we do not assume homogeneous data.

We summarize our main results in \cref{tab:results} and \cref{tab:results-nonsmooth}.
\begin{table}[!htb]
    \centering
    \caption{Comparison vs. SOTA for \textit{Smooth} Loss Functions. 
    \footnotesize{[LR'23] refers to \citet{lowy2023private}; we omit logs and fix \(M=N\), \(d = \Theta(n)\), \(\varepsilon = \Theta(1)\), \(L = \beta = D = 1\). See theorems and the appendix for more general cases.}}\label{tab:results}

\begin{tabular}{cccc}
    \toprule
    Algorithm \& Setting  & Excess Risk & Communication Complexity & Gradient Complexity \\
    \midrule
    {[LR'23]} Alg. 2 (i.i.d.) & optimal & 
    \(Nn\)\label{eq:lowy_bound} & \(N^2n + N^{3/2} n^{3/2}\) \\
    {[LR'23]} Alg. 1 (non-i.i.d.) & suboptimal & \(
        N^{1/5} n^{1/5} \)
        & \(Nn\)  \\
    Alg.~\ref{alg:phased_acc} (non-i.i.d.) & optimal & \(
        N^{1/4} n^{1/4}
    \) & \(N^{5/4} n^{1/4} + (N n)^{9/8}\)
    \\
    \bottomrule
    \end{tabular}
\end{table}

\begin{table}[!htb]
    \centering
    \caption{Comparison vs. SOTA for \textit{Nonsmooth} Loss Functions. First 3 algorithms use Nesterov smoothing and we omit the gradient complexity results for them (see \cref{omitted}).}\label{tab:results-nonsmooth}
    \begin{tabular}{cccc}
        \toprule
        Algorithm \& Setting & Excess Risk & Communication Complexity & Gradient Complexity \\
        \midrule
        {[LR'23]} Alg. 2 (i.i.d.) & optimal & 
        \(Nn \)
            & Omitted %
        \\
        {[LR'23]} Alg. 1 (non-i.i.d.) & suboptimal & \(
        N^{1/3} n^{1/3} \)
        & Omitted %
        \\
        Alg.~\ref{alg:phased_acc}  (non-i.i.d.)& optimal & \(N^{1/2} n^{1/2}\) & Omitted \\
        Alg.~\ref{alg:phased_acc_nonsmooth} (non-i.i.d.) & optimal & \(
        Nn\) & \( N^2 n + N^{3/2} n^{3/2}\) \\
        \cref{alg:phased_acc_nonsmooth_conv} (non-i.i.d.) & optimal & \(N^{1/2} n^{3/4}\) &
             \(N^{3/2} n^{3/4} + N^{5/4} n^{11/8}\) \\
        \bottomrule
    \end{tabular}
\end{table}
\paragraph{Our Algorithms and Techniques.}
Our algorithms combine and extend various private optimization techniques in novel ways.  

Our ISRL-DP \cref{alg:phased_acc} for smooth FL builds on the \textit{ISRL-DP Accelerated MB-SGD} approach used by \citet{lowy2023private} for federated empirical risk minimization (ERM). We call this algorithm repeatedly to  iteratively solve a carefully chosen sequence of regularized ERM subproblems, using the \textit{localization} technique of~\citet{feldmanPrivateStochasticConvex2020}. We obtain our novel optimal excess risk bounds for heterogeneous FL with an algorithmic \textit{stability}~\citep{bousquet2002stability} argument. The key observation is that the stability and generalization of regularized ERM in~\citet{shalev-shwartzStochasticConvexOptimization} does not require the data to be identically distributed. By combining this observation with a bound for ISRL-DP Accelerated MB-SGD, we obtain our excess risk bound.

It is not immediately clear that the approach just described should work, because the iterates of Accelerated MB-SGD are not necessarily stable. The instability of acceleration may explain why \citet{lowy2023private} used ISRL-DP Accelerated MB-SGD only for ERM and not for minimizing the population risk. We overcome this obstacle with an alternative analysis that leverages the stability of regularized ERM, the convergence of ISRL-DP Accelerated MB-SGD to an approximate minimizer of the regularized empirical loss, and localization. 
This argument enables us to show that our algorithm has optimal excess risk. By carefully choosing algorithmic parameters, we also obtain state-of-the-art communication and gradient complexities.   

To extend our algorithm to the nonsmooth case, we use two different techniques. One is Nesterov smoothing~\citep{nesterov2005smooth}, which results in an algorithm with favorable communication complexity. The other approach is to replace the ISRL-DP Accelerated MB-SGD ERM subsolver by an \textit{ISRL-DP Subgradient Method}. This technique yields an algorithm with favorable (sub)gradient complexity. Third, we use randomized convolution smoothing, as in~\citet{kulkarni2021private}, to obtain another favorable gradient complexity bound.

\subsection{Preliminaries}
\paragraph{Differential Privacy.} 
Let \(\mathbb{X} = \XX^{n \times N}\) and \(\rho:\mathbb{X}^2 \to [0, \infty)\) be a  distance between distributed data sets. Two distributed data sets \(\bx, \bx' \in \mathbb{X}\) are \textit{\(\rho\)-adjacent} if \(\rho(\bx, \bx') \leq 1\). Differential privacy (DP) prevents an adversary from distinguishing between the outputs of algorithm \(\alg\) when it is run on adjacent databases:
\begin{definition}[Differential Privacy~\citep{dwork2006calibrating}]\label{def: DP}
Let \(\varepsilon \geq 0, ~\delta \in [0, 1)\). A randomized algorithm \(\alg: \mathbb{X} \to \mathcal{W}\) is \textit{\((\varepsilon, \delta)\)-differentially private} (DP) for all \(\rho\)-adjacent data sets \(\bx, \bx' \in \mathbb{X}\) and all measurable subsets \(S \subseteq \WW\), we have 
\begin{equation}\label{eq: DP}
\mathbb{P}(\alg(\bx) \in S) \leq e^\varepsilon \mathbb{P}(\alg(\bx') \in S) + \delta.
\end{equation}
\end{definition}

\begin{definition}[Inter-Silo Record-Level Differential Privacy]\label{def: informal ISRL-DP}
Let \(\rho: \XX^{2n} \to [0, \infty)\), \(\rho(X_i, X'_i) := \sum_{j=1}^{n} \mathbbm{1}_{\{x_{i,j} \neq x_{i,j}'\}}\), \(i \in [N]\). 
A randomized algorithm \(\alg\) is \((\varepsilon, \delta)\)-ISRL-DP if for all \(i \in [N]\) and all \(\rho\)-adjacent silo data sets \(X_i, X'_i\), the full transcript of silo \(i\)'s sent messages 
satisfies (\ref{eq: DP}) 
for any fixed settings of other silos' data.
\end{definition}
\paragraph{Notation and Assumptions.} 
Let \(\|\cdot\|\) be the \(\ell_2\) norm and \(\Pi_{\mathcal{W}}(z):= \argmin_{w \in \mathcal{W}}\|w - z\|^2\) denote the projection operator. Function \(h: \mathcal{W} \to \mathbb{R}^m\) is \(L\)-Lipschitz if \(\|h(w) - h(w')\| \leq L \|w - w'\|\), \(\forall w, w' \in \mathcal{W}\). A differentiable function~\(h(\cdot)\) is \textit{\(\beta\)-smooth} if its derivative \(\nabla h\) is \(\beta\)-Lipschitz. 
For differentiable (w.r.t. \(w\)) \(f(w,x)\), we denote its gradient w.r.t. \(w\) by \(\nabla f(w,x)\).

We write \(a \lesssim b\) if \(\exists \, C > 0\) such that \(a \leq Cb\). We use \(a = \widetilde{O}(b)\) to hide poly-logarithmic factors.

Assume the following throughout:
\begin{assumption}\label{ass: basic}
\(\)\newline %
\vspace{-.2in}
\begin{enumerate}
     \item \(\mathcal{W} \subset \mathbb{R}^d\) is closed, convex. We assume \(\|w - w'\| \leq D\), \(\forall w, w' \in \mathcal{W}\).
    \item \(f(\cdot, x)\) is \(L\)-Lipschitz and convex for all \(x \in \mathcal{X}\). In some places, we assume that \(f(\cdot, x)\) is \(\beta\)-smooth. 
    \item In each round \(r\), a uniformly random subset \(S_r \subset [N]\) of \(M\) silos is available to communicate with the server.
\end{enumerate}
\end{assumption}
For simplicity, in the main body, we often assume \(M=N\). The Appendix contains the general statements of all results, complete proofs, and a further discussion of related work.  
\section{Localized ISRL-DP Accelerated MB-SGD for Smooth Losses}
We start with the smooth case.
Combining iterative localization techniques 
of~\citet{feldmanPrivateStochasticConvex2020,asi2021private} with the multi-stage ISRL-DP Accelerated MB-SGD of~\citet{lowy2023private}, our proposed~\cref{alg:phased_acc} achieves optimal excess risk and state-of-the-art communication complexity for heterogeneous FL under ISRL-DP.\@ 
\begin{algorithm}[!t]
    \caption{
    Localized ISRL-DP Accelerated MB-SGD
    }\label{alg:phased_acc}
    \begin{algorithmic}[1]
    \Require 
    Datasets \(X_l \in \mathcal{X}^n\) for \(l \in [N]\), loss function \(f\), constraint set \(\mathcal{W}\),
    initial point \(w_0\), subroutine parameters \(\{R_i\}_{i=1}^{\lfloor \log_2 n \rfloor} \subset \mathbb{N}\), \(\{K_i\}_{i=1}^{\lfloor \log_2 n \rfloor} \subset [n]\).
    \State Set \(\tau = \lfloor \log_2 n \rfloor\).
    \State Set \(p = \max(\tfrac{1}{2}\log_n(M) + 1, 3)\)
    \For{\(i = 1\) \textbf{to} \(\tau\)}
        \State Set \(\lambda_i = \lambda \cdot 2^{(i-1)p}\), \( n_i = \lfloor n / 2^i \rfloor\), \(D_i = 2L / \lambda_i\).
        \State Each silo \(l \in [N]\) draws disjoint batch \(B_{i,l}\) of \(n_i\) samples from \(X_l\).
        \State Let \(\hat F_i(w) = \frac{1}{n_i N} \sum_{l=1}^N \sum_{x_{l,j} \in B_{i,l}}  f(w; x_{l, j}) + \frac{\lambda_i}{2} \lVert w - w_{i-1} \rVert^2\).\label{line:ERM}
        \State Call the multi-stage \((\varepsilon, \delta)\)-ISRL-DP implementation of~\cref{alg:acc_mbsgd} 
        with loss function \(\hat F_i(w)\), data \(X_l = B_{i,l}\), \(R = R_i\), \(K = K_i\), initialization \(w_{i-1}\), constraint set \(\mathcal{W}_i = \{ w \in \mathcal{W} : \lVert w - w_{i-1} \rVert \leq D_i \}\), and \(\mu = \lambda_i\). Denote the output by \(w_i\).\label{line:multistg} 
    \EndFor
    \State \textbf{return} the last iterate \(w_{\tau}\)
\end{algorithmic}
\end{algorithm}
\begin{algorithm}[thb]
    \caption{Accelerated ISRL-DP MB-SGD~\citep{lowy2023private}}\label{alg:acc_mbsgd}
    \begin{algorithmic}[1]
    \Require Datasets \(X_l \in \mathcal{X}^{n}\) for \(l \in [N]\), loss function \(\hat{F}(w) = \frac{1}{nN}\sum_{l=1}^N\sum_{x \in X_l} f(w,x)\), constraint set \(\mathcal{W}\), initial point \(w_0\),
    strong convexity modulus \(\mu \geq 0\),
    privacy parameters \((\varepsilon, \delta)\),
    iteration count \(R \in \mathbb{N}\), batch size \(K \in [n]\), step size parameters \(\{\eta_r \}_{r \in [R]}, \{\alpha_r \}_{r \in [R]}\) specified in~\cref{app:multi-stage}.  
    \State  Initialize \(w_0^{ag} = w_0 \in \mathcal{W}\) and \(r = 1\).
    \For{\(r \in [R]\)}
    \State Server updates and broadcasts \\ \(w_r^{md} = \frac{(1- \alpha_r)(\mu + \eta_r)}{\eta_r + (1 - \alpha_r^2)\mu}w_{r-1}^{ag} + \frac{\alpha_r[(1-\alpha_r)\mu + \eta_r]}{\eta_r + (1-\alpha_r^2)\mu}w_{r-1}\)
    \For{\(i \in S_r\) \textbf{in parallel}} 
    \State Silo \(i\) draws \(\{x_{i,j}^r\}_{j=1}^K\) 
    from \(X_i\) (with replacement) and privacy noise \(u_i \sim \mathcal{N}(0, \sigma^2 \mathbf{I}_d)\) for proper \(\sigma^2\).
    \State Silo \(i\) computes \(\widetilde{g}_r^{i} := \frac{1}{K} \sum_{j=1}^{K} \nabla f(w_r^{md}, x_{i,j}^r) + u_i\).
    \EndFor
    \State Server aggregates \(\widetilde{g}_{r} := \frac{1}{M} \sum_{i \in S_r} \widetilde{g}_r^{i}\)
    and updates:
    \State \(w_{r} := \argmin_{w \in \mathcal{W}}\left\{\alpha_r \left[\langle \widetilde{g}_{r}, w\rangle + \frac{\mu}{2}\|w_r^{md} - w\|^2 \right] \right. \)
    \(\left. + \left[(1-\alpha_r) \frac{\mu}{2} + \frac{\eta_r}{2}\right]\|w_{r-1} - w\|^2\right\}\).
    \State Server updates and broadcasts \\
    \(
    w_{r}^{ag} = \alpha_r w_r + (1-\alpha_r)w_{r-1}^{ag}.
    \)
    \EndFor \\
    \State \textbf{return:} \(w_R^{ag}\).
\end{algorithmic}
\end{algorithm}

We recall ISRL-DP Accelerated MB-SGD in~\cref{alg:acc_mbsgd}. It is a distributed, ISRL-DP version of the AC-SA algorithm~\citep{ghadimilan1}. For strongly convex losses, the multi-stage implementation of ISRL-DP Accelerated MB-SGD, given in~\cref{alg:multistage} in~\cref{app:multi-stage}, offers improved excess risk~\citep{lowy2023private}. \cref{alg:multistage} is a distributed, ISRL-DP version of the multi-stage AC-SA~\citep{ghadimilan2}. 

Building on \cref{alg:multistage}, we describe our algorithm as follows (see~\cref{alg:phased_acc} for pseudocode).  The distributed learning process is divided into \(\tau = \lfloor \log_2 n \rfloor\) phases. In each phase \(i \in [\tau]\), all silos work together (via communication with the central server) to iteratively solve a regularized ERM problem with ISRL-DP.\@ The regularized ERM problem in phase \(i\) is defined over \(N\) local batches of data, each containing \(n_i\) \emph{disjoint} samples (cf.\ \(\hat{F}_i(w)\) in \cref{line:ERM}). To find the approximate constrained minimizer of \(\hat{F}_i(w)\) privately, we apply \cref{alg:multistage} for a careful choice of the number of rounds \(R_i\) and the batch size \(K_i \in [n_i]\) (cf. \cref{line:multistg}). 
The output \(w_{i}\) of phase \(i\) affects phase \(i+1\) in three ways: (i) regularization center used to define \(\hat{F}_{i+1}\); (ii) initialization for the next call of~\cref{alg:multistage}; and (iii) constraint set \(\WW_{i+1}\). We enforce \textit{stability} (hence generalization) of our algorithm via regularization and localization: e.g., as \(i\) increases, we increase the regularization parameter \(\lambda_i\) to prevent \(w_{i+1}\) from moving too far away from \(w_i\) and \(w^*\). 

The following theorem captures our main results of this section: 
\cref{alg:phased_acc} can achieve the optimal excess risk, \emph{regardless of the heterogeneity}, in a communication-efficient and gradient-efficient manner.

\begin{theorem}[Upper Bound for Smooth Losses]\label{thm:main}
    Let \(f(\cdot, x)\) be \(\beta\)-smooth and \(M=N\). Assume \(\varepsilon \leq 2  \ln(2/\delta), \delta \in (0,1)\). Then, there exist parameter choices such that \cref{alg:phased_acc} is \((\varepsilon,\delta)\)-ISRL-DP and has the following excess risk
    \begin{equation}\label{eq:main-excess_risk}
    \!\!\E F(w_{\tau}) \!-\! F(w^*) \!=\! \widetilde{O}\! \left( \frac{LD}{\sqrt{N}} \!\left(
        \frac{1}{\sqrt{n}} \!+\! \frac{\sqrt{d \log(1 / \delta)}}{\varepsilon n}
        \right)\!\!
        \right).
\end{equation}
Moreover, the communication complexity of~\cref{alg:phased_acc} is
    \begin{equation*}
    \widetilde O\left(
                 \frac{\sqrt{\beta D}N^{1/4}}{\sqrt{L}} \left(\min \left\{\sqrt{n}, \frac{\varepsilon n}{\sqrt{d \ln (1/\delta)}}\right\}\right)^{1/2} + 1
            \right).
\end{equation*}
Assuming \(d = \Theta(n)\) and \(\varepsilon = \Theta(1)\), the gradient complexity of~\cref{alg:phased_acc} is
\begin{align*}
    \widetilde{O}\left( N^{5/4} n^{1/4} (\beta D /L)^{1/2} + N n+  (N n)^{9/8}  (\beta D/L)^{1/4}\right).
\end{align*}
\end{theorem}
For general \(d, n, \varepsilon\), the gradient complexity expression is complicated, and is given in the \cref{sec:grad_complex}. 
\begin{remark}[Optimal risk in non-i.i.d private FL]
    The excess risk bound in~\eqref{eq:main-excess_risk} matches the optimal \textit{i.i.d} risk bound up to log factors (cf. Theorem 2.2 in~\citet{lowy2023private}), \emph{even when silo data is arbitrarily heterogeneous across silos}. To the best of our knowledge, our algorithm is the first to have this property, resolving an open question of~\citet{lowy2023private}. The prior state-of-the-art bound in~\citet[Theorem 3.1]{lowy2023private} is suboptimal by a factor of \(\widetilde{O}((\sqrt{d}/ (\varepsilon n \sqrt{N}))^{1/5})\). 
\end{remark}

\begin{remark}[Improved communication and gradient complexity]
The communication and gradient complexities of our~\cref{alg:phased_acc} significantly improve over the previous state-of-the-art for ISRL-DP FL: recall~\eqref{eq: LR smooth communication},~\eqref{eq: our R smooth},~\eqref{eq: LR23 smooth gradient complexity}, and~\eqref{eq: our smooth gradient complexity}.\end{remark}

Our algorithm has state-of-the-art communication complexity, \textit{even in the simple case of \(N=1\), where ISRL-DP FL reduces to central DP SCO}. In fact, the prior state-of-the-art iteration complexity bound for DP SCO was 
\(O((\beta D/L) \min \{\sqrt{n}, \varepsilon n/\sqrt{d\ln(1/\delta)}\})\)~\citep{zhangBringYourOwn2022}. By comparison, when \(N=1\), our algorithm's communication complexity is the square root of this bound. Note that when \(N=1\), our algorithm is essentially the same as the algorithm of~\citet{kulkarni2021private}, but we do not incorporate convolution smoothing here since we are assuming smoothness of the loss. 

We can also compare our \textit{gradient complexity} results against the state-of-the-art central DP SCO algorithms when \(N=1\). 
As an illustration, consider the interesting regime \(\varepsilon = \Theta(1)\) and \(d = \Theta(n)\). 
For smooth and convex losses, when (i) \(\beta \le O(\sqrt{n}L/D)\), algorithms in both \citet{feldmanPrivateStochasticConvex2020} and~\citet{zhangBringYourOwn2022} achieve optimal risk using \(\widetilde{O}(n)\) gradient complexity. For (ii) \(\beta \ge \Omega(\sqrt{n}L/D)\), the two algorithms proposed by~\citet{feldmanPrivateStochasticConvex2020} fail to guarantee optimal risk, whereas~\citet[Algorithm 2]{zhangBringYourOwn2022} continues to attain optimal risk with gradient complexity \(\widetilde{O}(n^{3/4}\sqrt{\beta D/L})\). By comparison with these results, for case (i), our \cref{alg:phased_acc} achieves optimal risk with gradient complexity \(\widetilde{O}(n^{5/4})\). For case (ii), as in~\citet{zhangBringYourOwn2022}, our \cref{alg:phased_acc} continues to achieve optimal risk with gradient complexity \(\widetilde{O}(n^{9/8} (\beta D/L)^{1/4} + n^{1/4} (\beta D/L)^{1/2} + n)\). Thus, our algorithm is faster than the state-of-the-art result of \citet{zhangBringYourOwn2022} when \(\beta D/L \geq n^{3/2}\). In the complementary parameter regime, however, the algorithm of~\citet{zhangBringYourOwn2022} (which is not ISRL-DP) is faster. We discuss possible ways to close this gap in Section~\ref{sec:conclude}.

\paragraph{Comparison with Non-Private Communication Complexity Lower Bound.}
\cref{thm:main} trivially extends to the unconstrained optimization setting in which \(D = \|w_0 - w^*\|\) for \(w^* \in \argmin_{w \in \mathbb{R}^d} F(w)\). Moreover, our excess risk bound is still optimal for the unconstrained case: a matching lower bound is obtained by combining the technique for unconstrained DP lower bounds in \citet{liu2021lower} with the constrained ISRL-DP FL lower bounds of \citet{lowy2023private}. 

Let us compare our communication complexity upper bound against the non-private communication complexity lower bound of~\citet{woodworth2020minibatch}. Define the following parameter  which describes the heterogeneity of the silos at the optimum \(w^* = \argmin_{w \in \mathbb{R}^d} F(w)\): \[
\zeta_*^2 = \frac{1}{N}\sum_{i=1}^N \|\nabla F_i(w^*)\|^2. 
\] 
The lower bound holds for the class of \textit{distributed zero-respecting} algorithms (defined in \cref{app: communication lower bound}), which includes most \textit{non-private} FL algorithms, such as MB-SGD, Accelerated MB-SGD, local SGD/FedAvg, and so on.
\begin{theorem}[Communication Lower Bound~\citep{woodworth2020minibatch}]\label{thm:lower-com}
Fix \(M=N\) and suppose \(\mathcal{A}\) is a distributed zero-respecting algorithm with excess risk \[
\E F(\mathcal{A}(X)) - F^* \lesssim \frac{LD}{\sqrt{N}}\left(\frac{1}{\sqrt{n}} + \frac{\sqrt{d \ln(1/\delta)}}{\varepsilon n}\right)\]
in \(\leq R\) rounds of communications on any \(\beta\)-smooth FL problem with heterogeneity \(\zeta_*\) and \(\|w_0 - w^*\| \leq D\). Then,
\begin{align*}
R &\gtrsim N^{1/4}\left(\min \left\{\sqrt{n}, \frac{\varepsilon n}{\sqrt{d \ln (1/\delta)}}\right\}\right)^{1/2}  \min\left(\frac{\sqrt{\beta D}}{\sqrt{L}}, \frac{\zeta_*}{\sqrt{\beta L D}}\right).
\end{align*}
\end{theorem}

\begin{remark} 
   Our communication complexity in \cref{thm:main} matches the lower bound on the communication cost in \cref{thm:lower-com} when \(\zeta_* \gtrsim \beta D\) (i.e., high heterogeneity).  
\end{remark}

There are several reasons why we cannot quite assert that~\cref{thm:lower-com} implies that~\cref{alg:phased_acc} is \textit{communication-optimal}. First, the lower bound does not hold for randomized algorithms that are not zero-respecting (e.g.\ our ISRL-DP algorithms). However, \citet{woodworth2020minibatch} note that the lower bound should be extendable to all randomized algorithms by using the random rotations techniques of~\citet{woodworth2016tight,carmon2020lower}. Second, the lower bound construction of~\citet{woodworth2020minibatch} is not \(L\)-Lipschitz. However, we believe that their quadratic construction can be approximated by a Lipschitz function (e.g.\ by using an appropriate Huber function). Third, as is standard in non-private complexity lower bounds, the construction requires the dimension \(d\) to grow with \(R\). This third issue seems challenging to overcome, and may require a fundamentally different proof approach. Thus, a rigorous proof of a communication complexity lower bound for ISRL-DP algorithms and Lipschitz functions is an interesting topic for future work. 
\paragraph{Sketch of the Proof of \cref{thm:main}.}
We end this section with a high-level overview of the key steps needed to establish \cref{thm:main}.

\begin{proof}[Proof sketch]
    \textbf{Privacy:} We choose parameters so that each call to \cref{alg:multistage} is \((\varepsilon, \delta)\)-ISRL-DP.\@ Then, the full \cref{alg:phased_acc} is \((\varepsilon, \delta)\)-ISRL-DP by parallel composition, since silo \(l\) samples \textit{disjoint} local batches across phases (\(\forall l \in [N]\)).
    
   \textbf{Excess risk:} Following~\citet{feldmanPrivateStochasticConvex2020,asi2021private}, we start with the following error decomposition: Let \(\hat w_0 = w^*\) for analysis only, and write
\begin{align}
     \E [F(w_{\tau})] - F(w^*) &= \E[F(w_{\tau}) - F(\hat{w}_{\tau})] \label{eq:t1}\\
     &+ \sum_{i=1}^{\tau} \E[F(\hat{w}_i) - F(\hat{w}_{i-1})]\label{eq:t2},
\end{align}
where \(\hat{w}_i=\argmin_{w \in \mathcal{W}} \hat F_i(w)\) for \(i \in [\tau]\). 
Then, it remains to bound~\eqref{eq:t1} and~\eqref{eq:t2}, respectively. To this end, we first show that \(w_i\) in \cref{alg:phased_acc} is close to \(\hat{w}_i\) for all \(i \in [\tau]\). Here, we mainly leverage the excess empirical risk bound of \cref{alg:multistage} for ERM with strongly convex and smooth loss functions. In particular, by the empirical risk bound in~\citet[Theorem F.1]{lowy2023private} and \(\lambda_i\)-strongly convex of \(\hat{F}_i\), we can show that the following bound holds for all \(i \in [\tau]\):
\begin{align}\label{eq:intermed}
    \E\left[\left\|w_i-\hat{w}_i\right\|^2\right] \le \widetilde O\left(\frac{L^2}{\lambda_i^2 N} \cdot \frac{d \log(1 / \delta)}{n_i^2 \varepsilon^2}\right).
\end{align}
We bound~\eqref{eq:t1} via~\eqref{eq:intermed}, Lipschitzness, and Jensen's Inequality.

To bound~\eqref{eq:t2}, we first leverage a key observation that \emph{stability and generalization of the empirical minimizer of strongly-convex loss does not require homogeneous data}. 
This observation that allows us to handle the non-i.i.d.\ case optimally. 
Specifically, by the stability result, we have 
\begin{align*}
     \E[F(\hat w_i) \!-\! F(\hat{w}_{i-1})] \lesssim  \frac{\lambda_i \E[\| \hat{w}_{i-1} - w_{i-1}\|^2]}{2} \!+\! \frac{L^2}{\lambda_i n_i M}.
\end{align*}
We obtain a bound on~\eqref{eq:t2} by combining the above inequality with~\eqref{eq:intermed}. Finally, by putting everything together and leveraging the geometrical schedule of \(\lambda_i, n_i\), we obtain the final excess population risk bound. 
\end{proof}

An alternative proof approach would be to try to show that~\cref{alg:multistage} is stable \textit{directly}, e.g.\ by using the tools that \citet{hardt2016train} use for proving stability of SGD. However, this approach seems unlikely to yield the same tight bounds that we obtain, since acceleration may impede stability of the iterates. Instead, we establish our stability and generalization guarantee \textit{indirectly}: We use the facts that regularized ERM is stable and that \cref{alg:multistage} approximates regularized ERM.

\section{Error-Optimal Heterogeneous ISRL-DP FL for Nonsmooth Losses}

In this section, we turn to the case of nonsmooth loss functions. We modify~\cref{alg:phased_acc}
to obtain algorithms that can handle nonsmooth losses. Our algorithms are the first to achieve the optimal excess risk for heterogeneous ISRL-DP FL with nonsmooth loss functions. 
\subsection{Reductions to the Smooth Case via Smoothing}
\paragraph{Nesterov Smoothing:}
Our first algorithm is based on the technique of Nesterov smoothing~\citep{nesterov2005smooth}: For a nonsmooth function \(f\), Moreau-Yosida regularization is used to approximate \(f\) by the \(\beta\)-smooth \(\beta\)-Moreau envelope \[
f_{\beta}(w):= \min_{v \in \mathcal{W}} \left(f(v) + \frac{\beta}{2}\|w - v\|^2 \right).
\]
We then optimize this smooth function using our \cref{alg:phased_acc}. See \cref{app:smoothing} for details.

\begin{theorem}[Nonsmooth FL via Nesterov smoothing]\label{thm:smoothing}
    Let \(M=N\). 
    Then the combination of \cref{alg:phased_acc} with Nesterov smoothing yields an \((\varepsilon,\delta)\)-ISRL-DP algorithm with optimal excess population risk as in~\eqref{eq:main-excess_risk}. The communication complexity is
    \begin{align*}
         \widetilde O\left(
                    \sqrt{N}
                    \min \left\{\sqrt{n}, \frac{\varepsilon n}{\sqrt{d \ln (1/\delta)}}\right\} + 1
            \right).
    \end{align*}
\end{theorem}

We have omitted the gradient complexity for this smoothing approach.\label{omitted}
This is mainly because the gradient computation of \(f_{\beta}\) needs additional computation in the form of the prox operator. One may consider using an \emph{approximate} prox operator to get a handle on the total gradient complexity as in~\citet{bassily2019private}. However, this approach would still introduce an additional 
\(\Theta(n^3)\) gradient complexity per step.
\paragraph{Convolutional Smoothing}
An alternative approach is to use convolutional smoothing, in which we approximate \(f(w, x_i)\) by the smooth function \(\E_{v \sim \mathcal{U}_s} f(w + v, x_i)\), where \(\mathcal{U}_s\) denotes the uniform distribution over the centered \(\ell_2\) ball of radius \(s\). We then apply \cref{alg:phased_acc} to the smooth approximation of $f$ with slight changes. We defer details to \cref{app:convolutional_smoothing} and only state simplfied results below.

\begin{theorem}[Nonsmooth FL via convolutional smoothing]
    Let \(M=N\). Then, there exist parameter choices such that combining \cref{alg:phased_acc} with convolutional smoothing (see \cref{alg:phased_acc_nonsmooth_conv}) yields an \((\varepsilon, \delta)\)-ISRL-DP algorithm with optimal excess population risk as in~\eqref{eq:main-excess_risk}.
    The communication complexity (in expectation) is
    \begin{equation}
        \widetilde O\left(\max\left\{1, 
            d^{1/4}\sqrt{M} \min \left\{\sqrt{n}, \frac{\varepsilon n}{\sqrt{d \ln (1/\delta)}}\right\}, \frac{\varepsilon^2 n}{d \ln(1/\delta)}
        \right\}\right).
    \end{equation}
    When \(d = \Theta(n)\), and \(\varepsilon = \Theta(1)\), the gradient complexity is
    \begin{equation*}
        \widetilde{O}\left( M^{3/2} n^{3/4} + M^{5/4} n^{11/8}\right).
    \end{equation*}
\end{theorem}
Compared with Nesterov smoothing, convolutional smoothing approach has better gradient complexity with slightly worse communication complexity. 

\begin{remark}[Optimal excess risk in non-i.i.d.\ private FL]
    By combining \cref{alg:phased_acc} with Nesterov smoothing or convolutional smoothing, we give the first optimal rate for nonsmooth heterogeneous FL under ISRL-DP.\@ Note that the previous best result for heterogeneous FL in~\citet{lowy2023private} is suboptimal and holds only for the \emph{smooth} case. One can combine the same smoothing technique above with the algorithm in~\citet{lowy2023private} to obtain a risk bound of 
    \(O(1/\sqrt{nN} + (\sqrt{d\ln(1/\delta)}/{(\varepsilon n \sqrt{N})})^{2/3})\)
    for the nonsmooth case, which is again suboptimal.  
\end{remark}

\begin{remark}[Improved communication complexity and gradient complexity]
The communication complexity of both smoothing approaches improves over the previous state-of-the-art result for an algorithm achieving optimal excess risk~\citep{lowy2023private}; recall~\eqref{eq: LR smooth communication}. The convolutional smoothing approach also improves the gradient complexity. Moreover, the result of~\citet{lowy2023private} assumed i.i.d.\ silo data, whereas our result holds for the non-i.i.d.\ case.
\end{remark}

Depending on the FL application, communication or computation may be more of a bottleneck. If communication efficiency is more of a priority, the Nesterov smoothing approach should be used. On the other hand, if computational efficiency is more pressing, convolutional smoothing is recommended.

\subsection{A Direct Subgradient Algorithm}
We propose another variation of~\cref{alg:phased_acc} that uses \textit{subgradients} to handle the nonsmooth case in a direct and computationally efficient way: see~\cref{alg:phased_acc_nonsmooth}. 
\cref{alg:phased_acc_nonsmooth} follows the same structure as~\cref{alg:phased_acc}: we iteratively solve a carefully chosen sequence of regularized ERM problems with a ISRL-DP solver and use localization.  Compared to \cref{alg:phased_acc} for the smooth case, \cref{alg:phased_acc_nonsmooth} does not use an accelerated solver (due to nonsmoothness). Instead, we use \textit{ISRL-DP Minibatch Subgradient} method (\cref{alg:MB-SGD}) to solve the nonsmooth strongly convex ERM problem in each phase of~\cref{alg:phased_acc_nonsmooth}. 
There are two key differences between our subroutine~\cref{alg:MB-SGD} and the ISRL-DP MB-SGD of~\citet{lowy2023private}: (i) Instead of the gradient, a subgradient of the nonsmooth objective is used in Line 5; (ii) A different and simpler averaging step in Line 10 is used for strongly convex nonsmooth losses. We state simplified results below; the complete version and the proof can be found in \cref{thm:nonsmooth-app}.
\begin{algorithm}[t]
    \caption{Noisy ISRL-DP MB-Subgradient Method}\label{alg:MB-SGD}
    \begin{algorithmic}[1]
    \Require 
    Datasets \(X_l \in \mathcal{X}^{n}\) for \(l \in [N]\), loss function \(\hat{F}(w) = \frac{1}{nN}\sum_{l=1}^N\sum_{x \in X_l} f(w,x)\), constraint set \(\mathcal{W}\), initial point \(w_0\),
    privacy parameters \((\varepsilon, \delta)\),
    iteration count \(R \in \mathbb{N}\), batch size \(K \in [n]\), step sizes \(\{\gamma_r\}_{r=0}^{R-1}\), initial point \(w_0 \in \WW\).
    \For{\(r \in \{0, 1, \ldots, R-1\}\)}
        \For{\(l \in S_r\) \textbf{in parallel}}
            \State Server sends global model \(w_r\) to silo \(l\).
            \State Silo \(l\) draws \(K\) samples \(x_{l,j}^r\) uniformly from \(X_i\) (for \(j \in [K]\)) and noise \(u_i \sim \mathcal{N}(0, \sigma^2 \mathbf{I}_d)\) for proper \(\sigma^2\).
            \State Silo \(l\) computes \(\widetilde{g}_r^{l} := \frac{1}{K} \sum_{j=1}^{K} g_{r,j}^{l} + u_i\) and sends to server, where \(g_{r,j}^{l} \in  \partial f(w_r, x_{l,j}^r)\) (subgradient)\label{line:subgradient}.
        \EndFor
        \State Server aggregates \(\widetilde{g}_{r} := \frac{1}{M_r} \sum_{l \in S_r} \widetilde{g}_r^{l}\).
        \State Server updates \(w_{r+1} := \Pi_{\mathcal{W}}[w_r - \gamma_r \widetilde{g}_r]\).
    \EndFor
    \State {\bfseries Output:} \(\bar{w}_R = \frac{2}{R(R+1)}\sum_{r=1}^R r w_r\)\label{line:average}.
\end{algorithmic}
\end{algorithm}

\begin{algorithm}[t]
    \caption{Localized ISRL-DP MB-Subgradient Method}\label{alg:phased_acc_nonsmooth}
    \begin{algorithmic}[1]
    \Require Dataset \(X_l \in \mathcal{X}^n, \, l \in [N]\), constraint set \(\mathcal{W}\), \(\eta > 0\), subroutine parameters (specified in Appendix) including batch size \(K_i\), number of rounds \(R_i\), noise parameters \(\sigma_i\).
    \State Choose any \(w_0 \in \WW\). 
    \State Set \(\tau = \lfloor \log_2 n \rfloor\), \(p = \max(\tfrac{1}{2}\log_n(M)+1, 3)\).
    \For{\(i = 1\) \textbf{to} \(\tau\)}
        \State Set \(\eta_i \!=\! \eta / 2^{i \cdot p}\), \(n_i \!=\! n / 2^i\), \(\lambda_i \!=\! 1 / ( \eta_i n_i)\), \(D_i \!=\! 2L / \lambda_i\).
        \State Each silo \(l \in [N]\) draws disjoint batch \(B_{i,l}\) of \(n_i\) samples from \(X_l\).
        \State Let \(\hat F_i(w) = \frac{1}{n_i N} \sum_{l=1}^N \sum_{x_{l,j} \in B_{i,l}}  f(w; x_{l, j}) + \frac{\lambda_i}{2} \lVert w - w_{i-1} \rVert^2\).
        \State Call the \((\varepsilon, \delta)\)-ISRL-DP~\cref{alg:MB-SGD} 
        with loss function \(\hat F_i(w)\), data \(X_l = B_{i,l}\), \(R = R_i\), \(K = K_i\), step sizes \(\gamma_r = \frac{2}{\lambda_i(r+1)}\) for \(r=0,1,\ldots, R_i -1\), initialization \(w_{i-1}\), and constraint set \(\mathcal{W}_i = \{ w \in \mathcal{W} : \lVert w - w_{i-1} \rVert \leq D_i \}\). Let \(w_i\) denote the output. 
    \EndFor
    \State \textbf{return} the last iterate \(w_{\tau}\).
\end{algorithmic}
\end{algorithm}

\begin{theorem}[Nonsmooth FL via subgradient]\label{thm:nonsmooth}
     Let \(M=N\). Then, there exist parameter choices such that \cref{alg:phased_acc_nonsmooth} is \((\varepsilon, \delta)\)-ISRL-DP and achieves the optimal excess population risk in~\eqref{eq:main-excess_risk}.
    The communication complexity is
    \begin{equation*}
        \widetilde{O} \left( \min \left(
        nN, \frac{N\varepsilon^2 n^2}{d}
    \right) + 1\right).
    \end{equation*}
    Assuming \(\varepsilon = \Theta(1)\), the subgradient complexity is
        \begin{align*}
        \widetilde{O} \Bigg(Nn +
        N^2 \min \left(n, \frac{n^2}{d}
            \right) + N^{3/2} \min \left(
                n^{3/2}, \frac{n^2}{\sqrt{d}} 
            \right)\Bigg).
    \end{align*}
\end{theorem}

\begin{remark}[Improved gradient complexity]
The above subgradient complexity improves over the previous state-of-the-art gradient complexity \citep{lowy2023private} for an ISRL-DP FL algorithm with optimal
excess risk. \citet{lowy2023private} apply Nesterov smoothing to ISRL-DP MB-SGD. As discussed
earlier, implementing the smoothing approach is computationally costly. Moreover, the results in \citep{lowy2023private} assume i.i.d. silo data.
\end{remark}

The precise statement and proof of~\cref{thm:nonsmooth} can be found in \cref{thm:nonsmooth-app}. 

\section{Numerical Experiments}\label{app:numerical}
We validate our theoretical findings with numerical experiments on MNIST data. As shown in Figures~\ref{fig:reliable} and~\ref{fig:unreliable}, \textit{our algorithm consistently outperforms \citet{lowy2023private}}. 
We use a similar experimental setup to~\citet{lowy2023private}, as outlined below.

\paragraph{Task/model/data set.} We run binary logistic regression with heterogeneous MNIST data: each of the \(N=25\) silos contains data corresponding to one odd digit class and one even digit class (e.g.\ (1, 0), (3, 2), etc.) and the goal is to classify each digit as odd or even. We randomly sample roughly \(1 / 5\) of MNIST data to expedite the experiments.
We borrow the code from \citet{woodworth2020minibatch} to transform and preprocess MNIST data.

\paragraph{Preprocessing.} We preprocess MNIST data, flatten the images, and reduce the dimension to \(d=50\) using PCA. We use an 80/20 train/test split, yielding a total of \(n=1734\) training samples.

\paragraph{Our algorithm.} Localized ISRL-DP MB-SGD, which is a practical (non-accelerated) variant of our \cref{alg:phased_acc}: For simplicity and to expedite parameter search, we use vanilla MB-SGD in place of accelerated MB-SGD as our regularized ERM subsolver in our implementation of \cref{alg:phased_acc}.
 
\paragraph{Baseline.} We compare our algorithm against the One-pass ISRL-DP MB-SGD of~\citet{lowy2023private}. Recall that~\citet{lowy2023private} did not provide theoretical guarantees for their multi-pass ISRL-DP MB-SGD with heterogeneous silos. 

\paragraph{Hyperparameter tuning and evaluation.} We evaluate the algorithms across a range of privacy parameters \(\varepsilon\) and fix \(\delta = 1 / n^2\). For each algorithm and each setting of \(\varepsilon\), we search a range of step sizes \(\eta\).

\paragraph{Simulating unreliable communication.} In addition to the reliable communication setting where all silos communicate in each round \(M=N=25\), we simulate unreliable communication by randomly selecting a subset of silos to communicate in each round.
In each communication round, \(M=18\) of the \(N=25\) silos are chosen uniformly at random to communicate with the server.

\paragraph{Evaluation.} Each evaluation consists of 5 trials, each with different data due to random sampling. In each trial, for each parameter setting, we repeat 3 runs and choose the hyperparameters with the lowest average loss, and record the average test error as the test error. We plot the average test error and the standard deviation across trials.

\begin{figure}[t]
    \centering
    \begin{minipage}[b]{0.45\textwidth}
        \centering
        \includegraphics[width=\textwidth]{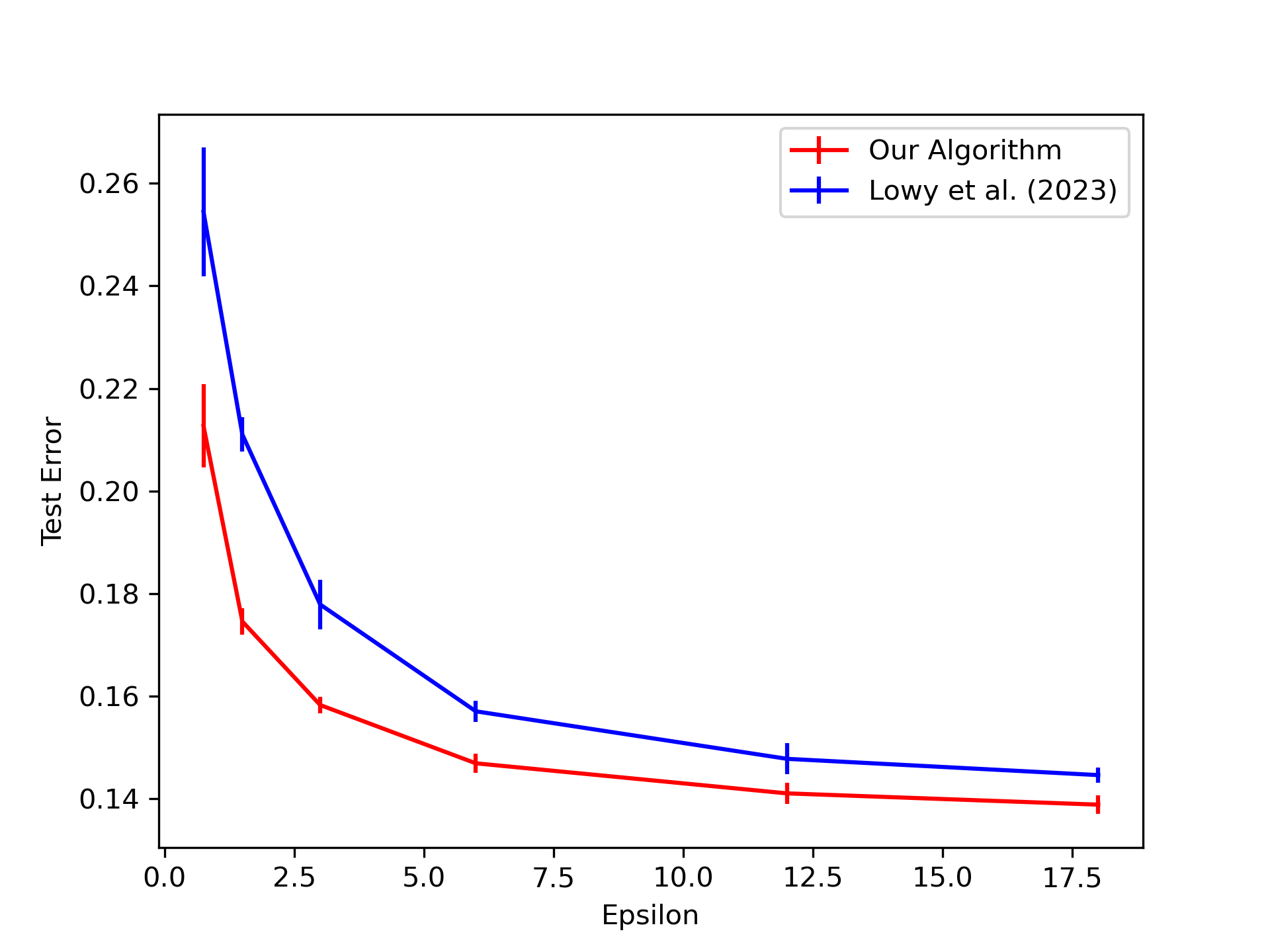}
        \caption{Reliable Communication}\label{fig:reliable}
    \end{minipage}
    \hfill
    \begin{minipage}[b]{0.45\textwidth}
        \centering
        \includegraphics[width=\textwidth]{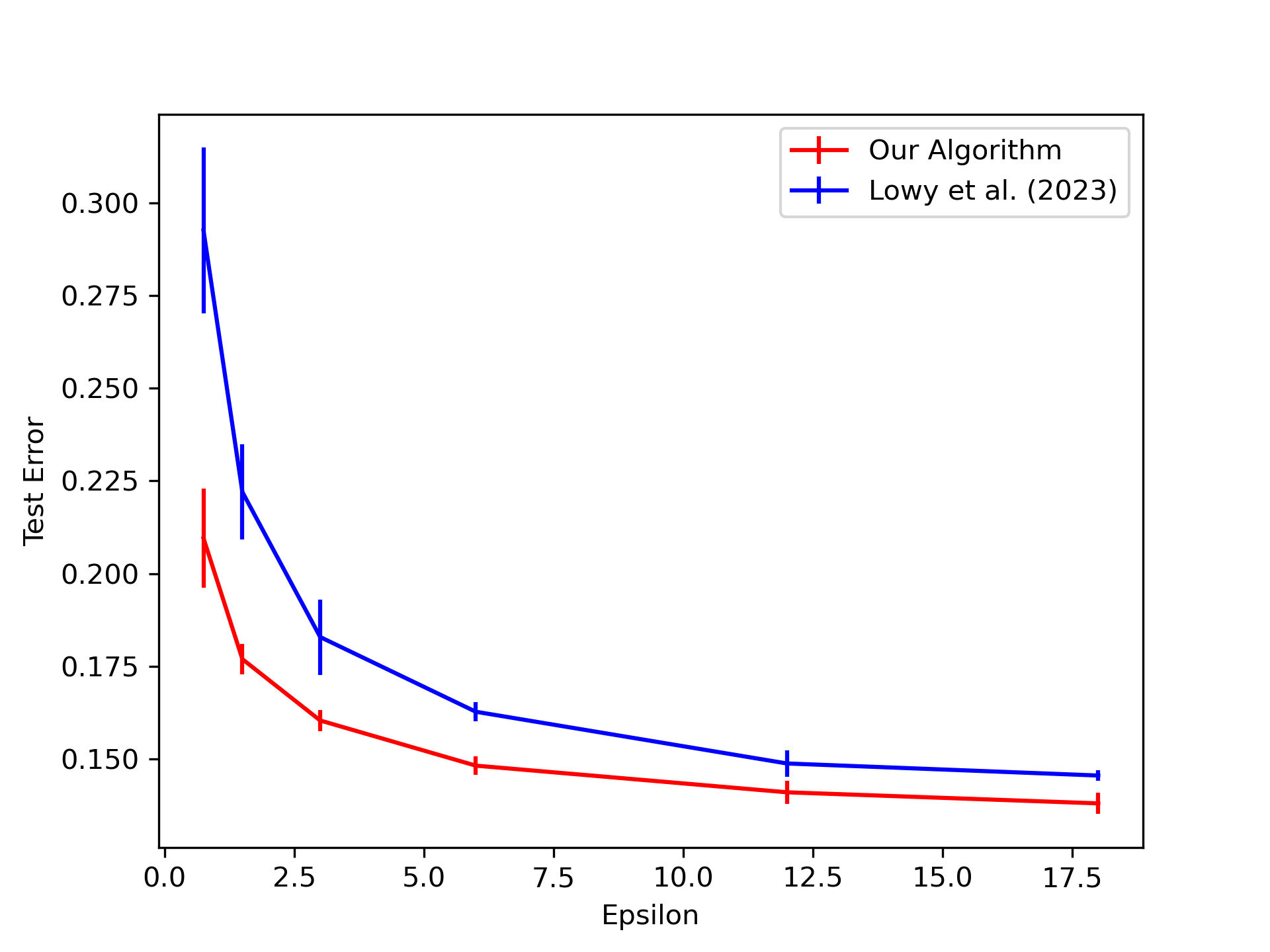}
        \caption{Unreliable Communication}\label{fig:unreliable}
    \end{minipage}
\end{figure}

As shown in the plots, \textit{in both reliable and unreliable communication settings, our localized ISRL-DP MB-SGD algorithm outperforms the baseline one-pass ISRL-DP MB-SGD algorithm across all privacy parameters}.
Despite being an algorithm designed to achieve theoretical guarantees, our algorithm evidently performs well in practice.

\section{Concluding Remarks and Open Questions}\label{sec:conclude}
We have studied private federated learning in the absence of a trusted server. We characterized the minimax optimal excess risk bounds for heterogeneous ISRL-DP FL, answering an open question posed by~\citet{lowy2023private}. Further, our algorithms advanced the state-of-the-art in terms of communication and computational efficiency.
For smooth losses, the communication complexity of our optimal algorithm matches the non-private lower bound.

To conclude, we discuss some open problems that arise from our work.
(1) A rigorous proof of a ISRL-DP communication complexity lower bound. 
(2) Is there an optimal ISRL-DP algorithm with \(O(nN)\) gradient complexity? A promising approach may be to combine~\cref{alg:phased_acc} with ISRL-DP variance-reduction. (Note that the gradient-efficient variance-reduced central DP algorithm of~\citet{zhangBringYourOwn2022} uses output perturbation, which requires a trusted server.)
(3) Is it possible to achieve both optimal communication complexity and optimal gradient complexity simultaneously with a single algorithm? Even in the simpler centralized setting (\(N=1\)), this question is open. 
(4) What are the optimal rates for ISRL-DP FL problems beyond convex and uniformly Lipschitz loss functions?
\section*{Acknowledgements}
XZ is supported in part by NSF CNS-2153220 and CNS-2312835. CG, AL, and SW supported in part by NSF CCF-2023239 and CCF-2224213 and AFOSR FA9550-21-1-0084.

\section*{Impact Statement}
In this research, we advance the field of federated learning (FL) by introducing algorithms that enhance privacy and communication efficiency in settings where trust in a central server is not assumed. Our work has significant implications for industries handling sensitive data, such as healthcare and finance, by offering an improved method for leveraging collective data while safeguarding individual privacy. However, the increased algorithmic complexity of these algorithms could limit their accessibility, especially for organizations with limited resources. Additionally, while our approach reduces the risk of data leakage, it does not entirely eliminate the possibility of misuse or unintended consequences, such as reinforcing existing biases in data. Future research should focus on refining these algorithms to make them more accessible and addressing potential ethical implications. By doing so, we aim to contribute positively to the development of safe and equitable data handling practices in various sectors of society and the economy.
 
\bibliography{refs}
\bibliographystyle{abbrvnat}

\appendix
\section{Further Discussion of Related Work}

\paragraph{DP Optimization.} There is a large and growing body of work on DP optimization. Most of this work focuses on the centralized setting, with Lipschitz convex loss functions in \(\ell_2\) geometry~\citep{bst14,bassily2019private,feldmanPrivateStochasticConvex2020,zhangBringYourOwn2022} and \(\ell_p\) geometry~\citep{asi2021private,bassily2021non}. Recently, we have started to learn more about other central DP optimization settings, such as DP optimization with non-uniformly Lipschitz loss functions/heavy-tailed data~\citep{lowy2023lipschitz}, non-convex loss functions~\citep{gao2023differentially,lowy2024make}, and min-max games~\citep{boob2023optimal}. There has also been work on the interactions between DP and other ethical desiderata, like fairness~\citep{lowy2023fair} and robustness~\citep{wu2023private}, as well as DP optimization with side access to public data~\citep{lowy2023public}. 
Despite this progress, much less is known about DP distributed optimization/federated learning, particularly in the absence of a trustworthy server. 

\paragraph{DP Federated Learning.} There have been many works attempting to ensure privacy of people's data during the federated learning (FL) process. Some of these works have utilized \textit{user-level differential privacy}~\citep{mcmahan17, geyer17, levy2021learning}, which can be practical for cross-device FL with a trusted server. Several works have also considered inter-silo record-level DP (ISRL-DP) or similar notions to ensure privacy without a trusted server~\citep{heikkila2020differentially, virginia, lowy2023private, lowy2023private,zhou2023differentially}. 

The state-of-the-art theoretical bounds for convex ISRL-DP FL are due to~\citet{lowy2023private} which gave minimax error-optimal algorithms and lower bounds for the i.i.d.\ setting, and suboptimal algorithms for the heterogeneous setting. We close this gap by providing optimal algorithms for the heterogeneous setting. Additionally, we improve over the communication complexity and gradient complexity bounds in~\citet{lowy2023private}.

\section{Multi-Stage Implementation of ISRL-DP Accelerated MB-SGD}\label{app:multi-stage}
\begin{algorithm}[htb]
    \caption{Multi-stage Accelerated Noisy MB-SGD~\citep{lowy2023private} 
    }\label{alg:multistage}
    \begin{algorithmic}[1]
    \Require Inputs: Constraint set \(\WW\), \(L\)-Lipschitz and \(\mu\)-strongly convex loss function \(\widehat{F}\), \(U \in [R]\) such that \(\sum_{k=1}^U R^{(k)} \leq R\) for \(R^{(k)}\) defined below; \(w_0 \in \mathcal{W}\), \(\Delta \geq \widehat{F}(w_0) - \widehat{F}^*\), 
    and \(q_0 = 0\).
    \For{\(k \in [U]\)}
        \State \(R^{(k)} = \left\lceil \max\left\{4 \sqrt{\frac{2\beta}{\mu}}, \frac{128 L^2}{3 \mu \Delta 2^{-(k+1)}} \right\} \right\rceil\)
        \State \(\upsilon_k \!=\! \max \left\{2 \beta,  \left[\frac{\mu V^2}{3 \Delta 2^{-(k-1)} R^{(k)} (R^{(k)} + 1)(R^{(k)} + 2)}\right]^{1/2}\right\}\)
        \State \(\alpha_r = \frac{2}{r + 1}\), \(\eta_r = \frac{4 \upsilon_k}{r(r+1)}\), for \(r \in [R^{(k)}]\). 
        \State Call \cref{alg:acc_mbsgd} with \(R=R^{(k)}\), using \(w_0 = q_{k-1}\), and \(\{\alpha_{r}\}_{r \in [R^{(k)}]}\) and \(\{\eta_{r}\}_{r \in [R^{(k)}]}\) defined above. 
        \State Set \(q_k\) to be the output of stage \(k\).
    \EndFor
    \State \textbf{return} \(q_U\).
\end{algorithmic}
\end{algorithm}

We will need the following result for the excess risk bound, which is due to~\citep{lowy2023private} (Theorem F.1).
\begin{lemma}[Smooth ERM Upper Bound for~\cref{alg:multistage}]\label{thm:lowy_erm_bound}
    Assume \(f(\cdot, x)\) is \(\beta\)-smooth and \(\lambda\)-strongly convex for all \(x\). Let \(\varepsilon \leq 2  \ln(2/\delta), \delta \in (0,1)\). Then, there exist algorithmic parameters such that \cref{alg:multistage} is \((\varepsilon, \delta)\)-ISRL-DP.\@ Moreover, \cref{alg:multistage} has the following excess empirical risk bound
    \begin{equation}\label{eq: sc accel smooth ERM upper}
    \E [\widehat{F}(q_U) - \widehat{F}^*] = \widetilde{O}
    \left(\frac{L^2}{\lambda}\frac{d 
    \ln(1/\delta)
    }{\varepsilon^2 n^2 M}\right),
    \end{equation}
    and the communication complexity is
    \begin{equation}\label{eq:R}
        R =\max\left\{1, \sqrt{\frac{\beta}{\lambda}} \ln\left(\frac{\Delta \lambda M \varepsilon^2 n^2}{L^2 d}\right), \mathbbm{1}_{\{M K < Nn\}} \frac{\varepsilon^2 n^2}{K d \ln(1/\delta)}\right\}. 
    \end{equation}
\end{lemma}

\section{Precise Statement and Proof of \texorpdfstring{\cref{thm:main}}{\cref{thm:main}}}\label{proof:main}
\begin{theorem}[Precise Statement of \cref{thm:main}]\label{thm:main-app}
Assume \(f(\cdot, x)\) is \(\beta\)-smooth for all \(x\). Let \(\varepsilon \leq 2  \ln(2/\delta)\), \(\delta \in (0,1)\). Choose \(R_i \approx \max \left(\sqrt{\frac{\beta + \lambda_i}{\lambda_i}} \ln\left(\frac{\Delta_i \lambda_i M \varepsilon^2 n_i^2}{L^2 d}\right),
\mathbbm{1}_{\{M K_i  < Nn_i\}} \frac{\varepsilon^2 n_i^2}{K_i d \ln(1/\delta)}\right)\), where \(LD \ge \Delta_i \ge \hat F_i (w_{i-1}) - \hat F_i(\hat w_i)\), \(K_i \geq \frac{\varepsilon n_i}{4 \sqrt{2R_i \ln(2/\delta)}}\), \(\sigma_i^2 = \frac{256 L^2 R_i \ln(\frac{2.5 R_i}{\delta}) \ln(2/\delta)}{n_i^2 \varepsilon^2}\), and \begin{equation}\label{eq:lambda}
    \lambda = \frac{L}{D n\sqrt{M}} \max \left\{\sqrt{n}, \frac{\sqrt{d \ln (1/\delta)}}{\varepsilon}\right\}.
\end{equation}
Then, the output of \cref{alg:phased_acc} is \((\varepsilon, \delta)\)-ISRL-DP and achieves the following excess risk bound:
\begin{equation}\label{eq:excess_risk}
    \E F(w_{\tau}) - F(w^*) = \widetilde{O}\left( \frac{LD}{\sqrt{M}} \left(
        \frac{1}{\sqrt{n}} + \frac{\sqrt{d \log(1 / \delta)}}{\varepsilon n}
        \right)
        \right).
\end{equation}
The communication complexity is
\begin{equation}\label{eq:smooth_comm}
    \widetilde O\left(\max \left\{1, 
                 \frac{\sqrt{\beta D}M^{1/4}}{\sqrt{L}} \left(\min \left\{\sqrt{n}, \frac{\varepsilon n}{\sqrt{d \ln (1/\delta)}}\right\}\right)^{1/2},
                 \mathbbm{1}_{\{M < N\}} \frac{\varepsilon^2 n}{d \ln(1/\delta)}
                 \right \}
            \right),
\end{equation}
when \(K_i = n_i\).
If \(d = \Theta(n)\), \(M=N\), and \(\varepsilon = \Theta(1)\), then the gradient complexity is
\begin{equation*}
     \widetilde{O}\left( N^{5/4} n^{1/4} (\beta D /L)^{1/2} +   N n +  (N n)^{9/8}  (\beta D/L)^{1/4}\right).
\end{equation*}
\end{theorem}

\begin{remark}
    In \cref{thm:main}, we state the results only for the case \(M=N\). Here, we do not assume \(M=N\), and present the complete results. 
    The complete analysis of gradient complexity for other regimes can be found in \cref{sec:grad_complex}.
\end{remark}
Before proving the theorem, we need several lemmas.

1. We first show the following result.
\begin{lemma}\label{lem:Fhat_property}
Let \(\hat{w}_i=\argmin_{w \in \mathcal{W}} \hat F_i(w)\).
We have \(\hat{w}_i \in \mathcal{W}_i\) and \(\hat F_i\) is \(3L\)-Lipschitz, \((\beta + \lambda_i)\)-smooth.
\end{lemma}
\begin{proof}
The optimality of \(\hat{w}_i\) implies that
\[
    \frac{1}{n_i M} \sum_{l=1}^M \sum_{j=1}^{n_i} f(\hat w_i; x_{l, j}) + \frac{\lambda_i}{2} \left\|\hat w_i-w_{i-1}\right\|^2 \le
    \frac{1}{n_i M} \sum_{l=1}^M \sum_{j=1}^{n_i} f(w_{i-1}; x_{l, j})  + 0.
\]
By rearranging and using the \(L\)-Lipschitzness of \(f(\cdot, x)\), We obtain
\[
\frac{\lambda_i}{2} \left\|\hat{w}_i-w_{i-1}\right\|^2 \leq L\left\|\hat{w}_i-w_{i-1}\right\|.
\]
It follows that \(\hat{w}_i \in \mathcal{W}_i=\left\{w:\left\|w-w_{i-1}\right\| \leq \frac{2 L}{\lambda_i}\right\}\).

For Lipschitzness, the norm of the derivative of the regularizer \(r_i(w) = \frac{\lambda_i}{2} \left\|w-w_{i-1}\right\|^2\) is \(\lambda_i \left\|w-w_{i-1}\right\|\), which is bounded by \(\le \lambda_i D_i = 2L\).
The Hssian of the regularizer is \(\lambda_i I\).

Therefore, \(r_i(w)\) is \(2L\)-Lipschitz and \(\lambda_i\)-smooth.
It follows that \(\hat F_i\) is \(3L\)-Lipschitz and \((\beta + \lambda_i)\)-smooth.
\end{proof}

2. We have the following bounds that relate the private solution \(w_i\) and the true solution \(\hat w_i\) of \(\hat F_i\).
\begin{lemma}\label{lem:iter_risk}
    In each phase \(i\), the following bounds hold:
    \begin{equation}\label{eq:Fi_excess_risk}
        \E [\hat F_i(w_i) - \hat F_i(\hat w_i)] = \widetilde{O}\left(\frac{L^2}{\lambda_i M}\cdot \frac{d \ln(1/\delta)}{\varepsilon^2 n_i^2} \right),
    \end{equation}
    \begin{equation}\label{eq:w_dist}
        \E\left[\left\|w_i-\hat{w}_i\right\|^2\right] \le \widetilde O\left(\frac{L^2}{\lambda_i^2 M} \cdot \frac{d \log(1 / \delta)}{n_i^2 \varepsilon^2}\right).
    \end{equation}
\end{lemma}

\begin{proof}
Applying \cref{thm:lowy_erm_bound} to \(\hat F_i\), we have
\begin{equation*}
    \E [\hat F_i(w_i) - \hat F_i(\hat w_i)] = \widetilde{O}\left(\frac{L^2}{\lambda_i M}\cdot \frac{d \ln(1/\delta)
    }{\varepsilon^2 n_i^2} \right).
\end{equation*}
By using the \(\lambda_i\)-strong convexity, we have
\begin{equation*}
    \frac{\lambda_i}{2} \E\left[\left\|w_i-\hat{w}_i\right\|^2\right] \le \E [\hat F_i(w_i) - \hat F_i(\hat w_i)].
\end{equation*}
The bound~\eqref{eq:w_dist} follows.
\end{proof}
3. As a consequence, we have the following bound.

\begin{lemma}\label{lem:excess_risk_wi}
    Let \(w \in \mathcal{W}\). We have
    \begin{equation*}
        \E[F(\hat w_i)] - F(w) \le \frac{\lambda_i \E[\| w - w_{i-1}\|^2]}{2} + \frac{4 \cdot (3L)^2}{\lambda_i n_i M}.
    \end{equation*}
\end{lemma}
\begin{proof}
Applying the stability result in \cref{lem:stab} to \(\hat F_i\) with \(m = M n_i\), which is \(\lambda_i\)-strongly convex and \(3L\)-Lipschitz, we have
\begin{equation*}
    \E[\hat F_i(\hat w_i) - \hat F_i(w)] \le \frac{4 \cdot (3L)^2}{\lambda_i n_i M}.
\end{equation*}
It follows from the definition of \(\hat F_i\) that
\begin{equation} \label{eq:empirical_excess_risk}
    \begin{aligned}
        \E[F(\hat w_i)] - F(w)
        & = \E[\hat F_i(\hat w_i)] - \frac{\lambda_i \E[\|\hat w_i - w_{i-1}\|^2]}{2} - \left(\hat F_i(w) - \frac{\lambda_i \E[\| w - w_{i-1}\|^2]}{2}\right) \\
        & \le \frac{\lambda_i \E[\| w - w_{i-1}\|^2]}{2} + \E[\hat F_i(\hat w_i) - \hat F_i(w)] \\
        & \le \frac{\lambda_i \E[\| w - w_{i-1}\|^2]}{2} + \frac{4 \cdot (3L)^2}{\lambda_i n_i M}.
    \end{aligned}
\end{equation}
\end{proof}

Putting these results together, we now prove the theorem.
\begin{proof}[Proof of \cref{thm:main-app}]
\textbf{Privacy.}  By the privacy guarantee of~\cref{alg:multistage} given in~\cref{thm:lowy_erm_bound}, each phase of the algorithm is \((\varepsilon, \delta)\)-ISRL-DP.\@ Since the batches \(\{B_{i,l}\}_{i=1}^{\tau}\) are disjoint for all \(l \in [N]\), the privacy guarantee of the entire algorithm follows by parallel composition of differential privacy~\citep{mcsherry2009privacy}.

\textbf{Excess risk.}
Recall that we define \(\hat w_0 = w^*\). Write
\begin{equation*}
    \E F(w_{\tau}) - F(w^*) = \E[F(w_{\tau}) - F(\hat{w}_{\tau})] + \sum_{i=1}^\tau \E[F(\hat{w}_i) - F(\hat{w}_{i-1})].
\end{equation*}
Since \(\tau = \lfloor \log_2 n \rfloor\), we have \(n_{\tau} = \Theta(1)\) and \( \lambda_{\tau} = \Theta(\lambda n^p)\).
By~\eqref{eq:w_dist} and Jensen's Inequality \(\E Z \le \sqrt{\E Z^2} \), we bound the first term as follows:
\begin{equation*}
    \begin{aligned}
        \E[F(w_{\tau}) - F(\hat{w}_{\tau})] \le L \E \left[\|w_{\tau} - \hat{w}_{\tau}\|\right]
        & \le \widetilde{O} \left(\frac{L^2}{\sqrt{M} \lambda_{\tau} n_{\tau}} \cdot \frac{\sqrt{d \log (1 / \delta)}}{\varepsilon}\right) \\
        & \le \widetilde{O} \left(\frac{L^2}{\lambda n^p \sqrt{M}} \cdot \frac{\sqrt{d \log (1 / \delta)}}{\varepsilon}\right) \\
        & \le \widetilde{O} \left(\frac{L^2}{\frac{L}{D\sqrt{nM}} \cdot n^p \sqrt{M}} \right) \le \widetilde{O} \left(\frac{DL}{n^{p-\tfrac{1}{2}}}\right),
    \end{aligned}
\end{equation*}
where the last step is due to the choice of \(\lambda\), per~\eqref{eq:lambda}.
Now recall \(p = \max(\tfrac{1}{2}\log_n(M) + 1, 3)\). We have
\begin{equation}\label{eq:p}
    n^{p-\tfrac{1}{2}} \ge \sqrt{n \cdot n^{\log_n(M)}} = \sqrt{nM}.
\end{equation}
It follows that \(\E[F(w_{\tau}) - F(\hat{w}_{\tau})] \le \widetilde{O}\left(\frac{DL}{\sqrt{nM}}\right)\).

Note that \(\lambda_i n_i^2 = \Theta(\lambda n^2 \cdot 2^{(p-2)i})\) and \(p \ge 3\). We know that \(\lambda_i n_i^2 \) and \(\lambda_i n_i\) increase geometrically.
By combining \cref{lem:excess_risk_wi} and~\eqref{eq:Fi_excess_risk}, we obtain
\begin{equation*}
    \begin{aligned}
        \sum_{i=1}^{\tau} \E[F(\hat{w}_i) - F(\hat{w}_{i-1})]
        & \le \sum_{i=1}^{\tau} \left(\frac{\lambda_i \E[\| w_{i-1} - \hat w_{i-1}\|^2]}{2} + \frac{4 \cdot (3L)^2}{\lambda_i n_i M}\right) \\
        & \le \widetilde O\left(\lambda D^2 +
            \sum_{i=2}^\tau \frac{L^2}{\lambda_i n_i^2 M} \cdot \frac{d \log(1 / \delta)}{\varepsilon^2} + \sum_{i=1}^\tau \frac{L^2}{\lambda_i n_i M}
        \right) \\
        & \le \widetilde O\left( \lambda D^2 +
        \frac{L^2}{\lambda n^2 M} \cdot \frac{d \log(1 / \delta)}{\varepsilon^2} + \frac{L^2}{\lambda n M}
        \right),
    \end{aligned}
\end{equation*}
setting \(\lambda\) as per~\eqref{eq:lambda} gives the result.

\textbf{Communication complexity.} When we use the full batch in each round, that is, \(K_i = n_i\), hiding logarithmic factors, communication complexity is 
\begin{equation}\label{eq:comm_cost}
    \begin{aligned}
        \sum_{i=1}^{\tau} R_i
        & = \sum_{i=1}^{\tau} \max\left\{1, \sqrt{\frac{\beta + \lambda_i}{\lambda_i}} \ln\left(\frac{\Delta \lambda_i M \varepsilon^2 n_i^2}{L^2 d}\right), \mathbbm{1}_{\{M < N\}} \frac{\varepsilon^2 n_i^2}{n_i d \ln(1/\delta)}
        \right\} \\
        & = \widetilde O \left(\max\left\{\lfloor \log_2 n \rfloor, \sqrt{\frac{\beta}{\lambda}}, \mathbbm{1}_{\{M < N\}} \frac{\varepsilon^2 n}{d \ln(1/\delta)}
        \right\} \right)\\
        & = \widetilde O\left(\max\left\{1, 
                 \frac{\sqrt{\beta D}M^{1/4}}{\sqrt{L}} \left(\min \left\{\sqrt{n}, \frac{\varepsilon n}{\sqrt{d \ln (1/\delta)}}\right\}\right)^{1/2}, \mathbbm{1}_{\{M < N\}} \frac{\varepsilon^2 n}{d \ln(1/\delta)}
            \right\}\right).
    \end{aligned}
\end{equation}

For gradient complexity, we defer the analysis to the following section.
\end{proof}

\subsection{Gradient Complexity of \texorpdfstring{\cref{alg:phased_acc}}{\cref{alg:phased_acc}} under Different Parameter Regimes}\label{sec:grad_complex}
\begin{theorem}
When \(N=M\) and the full batch \(K_i = n_i\) is used, the gradient complexity of~\cref{alg:phased_acc} is
\begin{equation}\label{eq:grad_cost}
    \begin{aligned}
        \sum_{i=1}^{\tau} N n_i R_i
        & = \widetilde O \left(Nn \max\left\{1, \sqrt{\frac{\beta}{\lambda}}\right\}\right) \\
        & = \widetilde O\left(Nn +
            \frac{n \cdot \sqrt{\beta D}N^{5/4}}{\sqrt{L}} \min \left\{\sqrt{n}, \frac{\varepsilon n}{\sqrt{d \ln (1/\delta)}}\right\}^{1/2}
       \right).
    \end{aligned}
\end{equation}

When we choose \(K_i < n_i\), the communication cost can be worse due to the second term of \(R_i = \max \left\{1, \sqrt{\frac{\beta + \lambda_i}{\lambda_i}} \ln\left(\frac{\Delta_i \lambda_i M \varepsilon^2 n_i^2}{L^2 d}\right),
\mathbbm{1}_{\{M K_i  < Nn_i\}} \frac{\varepsilon^2 n_i^2}{K d \ln(1/\delta)}\right\}\).
However, under some regimes, the gradient complexity can be better than that of the full batch. In the case where \(M < N\), the second case will always be present. Now we relax the assumption \(M=N\), and discuss general results.

For simplicity, let us assume \(\varepsilon = \Theta(1)\), keep terms involving \(\varepsilon\), \(M\), \(n\), and \(d\) only and omit \(\widetilde O\).
We summarize the results as follows: 
\begin{itemize}
\item \sloppy When \(d \lesssim n\), there are two subcases to consider: if \(d \gtrsim \frac{n^{3/4}}{M^{1/4}}\), then gradient complexity is \(\min \left\{(M + M^{5/4} n^{1/4})\max\left(\frac{n^{7/4}}{M^{1/4}d}, \frac{n^{7/8}}{M^{1/8}}, 1\right), \frac{Mn^2}{d}\right\}\). If \(d \lesssim \frac{n^{3/4}}{M^{1/4}}\), the complexity is \(\frac{Mn^2}{d}\). %
\item \sloppy When \(d \gtrsim n\), there are two subcases to consider: if \(d \gtrsim \frac{n^{2/3}}{M^{1/3}}\), then gradient complexity is \(\min\left( \left(M + \frac{M^{5/4} n^{1/2}}{d^{1/4}}\right) \max\left(\frac{n^{3/2}}{M^{1/4}d^{3/4}}, \frac{n^{3/4}d^{1/8}}{M^{1/8}}, 1\right), Mn \right)\). If \(d \lesssim \frac{n^{2/3}}{M^{1/3}}\), the complexity is \(Mn\). %
\end{itemize}
In particular, when \(d = \Theta(n)\), the gradient complexity is \(M^{5/4} n^{1/4} + (M n)^{9/8}\), or more precisely,
\begin{equation}\label{eq:grad_cost_d_eq_n}
    \widetilde{O}\left( N^{5/4} n^{1/4} (\beta D /L)^{1/2} + N n+  (N n)^{9/8}  (\beta D/L)^{1/4}\right)
\end{equation}
\end{theorem}
\begin{proof}
The bound ~\eqref{eq:grad_cost} is an immediate consequence of~\cref{thm:main}. 

Since terms in the sum decrease geometrically, we only need to consider the first term where \(i=1\). For simplicity, we drop the index \(i\). 

We write \(R = \max(R', R'') + 1\), where we let \(R' = Q \cdot \left(\min \left\{\sqrt{n}, \frac{\varepsilon n}{\sqrt{d \ln (1/\delta)}}\right\}\right)^{1/2}\), \(R'' =\frac{\varepsilon^2 n^2}{K d \ln(1/\delta)}\), and \(Q := \sqrt{\beta D} M^{1/4} / \sqrt{L}\).  Recall the batch size constraint \(K \ge \frac{\varepsilon n}{4 \sqrt{2R\ln(2/\delta)}}\). We analyze the complexity for (1) \(d \lesssim n\), (2) \(d \gtrsim n\) separately below. We omit \(\widetilde O\) for simplicity.

1. When \(d \lesssim n\), we have \(\min \left\{\sqrt{n}, \frac{\varepsilon n}{\sqrt{d \ln (1/\delta)}}\right\} = O\left(\sqrt{n}\right)\).

1a. If \(K \gtrsim \frac{\varepsilon^2 n^{7/4}}{Qd\ln(1/\delta)} = K_a\), we have \(R' \gtrsim R''\) and \(R = R' + 1\). The constraint on the batch size simplifies to \(K \gtrsim \frac{\varepsilon n}{\sqrt{Qn^{1/4}\ln(1/\delta)}} =: K_b\) when \(K_a \lesssim n\).
Together, the gradient complexity is \(M(R'+1) \max(K_a, K_b, 1)\).

1b. Otherwise if \(K \lesssim K_a\), we have \(R = R'' + 1\). The constraint on the batch size simplifies to \(K \gtrsim \sqrt{Kd}\). Therefore, we need \(d \lesssim K < n\), which is already assumed. In this case, \(K=d\), and the gradient complexity is \(M(R''+1) \cdot K = Md + M \cdot \frac{\varepsilon^2 n^2}{d \ln(1/\delta)}\).
2. When \(d \gtrsim n\), we have \(\min \left\{\sqrt{n}, \frac{\varepsilon n}{\sqrt{d \ln (1/\delta)}}\right\} = O\left(\frac{\varepsilon n}{\sqrt{d \ln (1/\delta)}}\right)\).

2a. If \(K \gtrsim \frac{(\varepsilon n)^{3/2}}{Q(d\ln(1/\delta))^{3/4}} =: K_c\), we have \(R' \gtrsim R''\) and \(R = R' + 1\).
The constraint on the batch size reduces to \(K \gtrsim  \frac{(\varepsilon n)^{3/4} d^{1/8}}{Q^{1/2}(\ln(1/\delta))^{3/8}} =: K_d\). In the case when \(K_d \ge n\), we will use the full batch. Therefore, the gradient complexity is \(M(R'+1) \max(K_c, \min(K_d, n), 1)\) if \(K_c \lesssim n\).

2b. If \(K \lesssim K_c\), we have \(R = R'' + 1\). As in the case 1b., we need \(d \lesssim n\), which contradicts with our assumption \(d = \Omega(n)\). In this case, we need to use the full batch. We have \(K=n\) and the gradient complexity is \(M(R''+1) \cdot K = Mn + M \cdot \frac{\varepsilon^2 n^2}{d\ln(1/\delta)}\).

We summarize as follows.

1. When \(d \lesssim n\), the gradient complexity is \(\min \left\{M(R' + 1) \max(K_a, K_b, 1),Md + M \cdot \frac{\varepsilon^2 n^2}{d \ln(1/\delta)}\right\}\) if \(K_a \lesssim n\), and  \(Md + M \cdot \frac{\varepsilon^2 n^2}{d \ln(1/\delta)}\) if \(K_a \gtrsim n\).

2. When \(d \gtrsim n\), the gradient complexity is \(\min \left\{M(R'+1) \max(K_c, \min(K_d, n), 1), Mn + M \cdot \frac{\varepsilon^2 n^2}{d \ln(1/\delta)}\right\}\) if \(K_c \lesssim n\), and  \(Mn + M \cdot \frac{\varepsilon^2 n^2}{d\ln(1/\delta)}\) if \(K_c \gtrsim n\).

Keeping terms involving \(\varepsilon\), \(M\), \(n\), and \(d\) only, we have \(Q = M^{1/4}\), \(R' = M^{1/4} \min\{n^{1/4}, \frac{n^{1/2}}{d^{1/4}}\}\),
\(K_a = \frac{n^{7/4}}{M^{1/4}d}\), \(K_b = \frac{n^{7/8}}{M^{1/8}}\), \(K_c = \frac{n^{3/2}}{M^{1/4}d^{3/4}}\), \(K_d = \frac{n^{3/4}d^{1/8}}{M^{1/8}}\). 
The results follow by plugging in the expressions. It is straightforward to verify the complexity for the special case \(d = \Theta(n)\).
\end{proof}
\section{Optimize Nonsmooth Losses via Convolutional Smoothing}\label{app:convolutional_smoothing}

\subsection{Preliminaries for Convolutional Smoothing}
We first provide a brief overview of convolutional smoothing and its key properties. For simplicity, let \(\mathcal{U}_s\) denote the uniform distribution over the \(\ell_2\) ball of radius \(s\).

\begin{definition}[Convolutional Smoothing]
    For ERM with convex loss \(f\), that is \(\hat F(w) = \frac{1}{n} \sum_{i=1}^n f(w, x_i)\). We define the convolutional smoother of \(f\) with radius \(s\) as
    \(\tilde f_s := \E_{v \sim \mathcal{U}_s} f(w + v, x_i)\). Then the ERM smoother is defined accordingly, as follows, \(\hat F_{s}(w) = \frac{1}{n} \sum_{i=1}^n \E_{v \sim \mathcal{U}_s} f(w + v, x_i)\).
\end{definition}
    
We have the following properties of the smoother \(\hat F_{n_s}\)~\citep{kulkarni2021private}:
\begin{lemma}\label{lem:conv_smooth_properties}
Suppose \(\{f(\cdot, x)\}_{x \in \Xi}\) is convex and \(L\)-Lipschitz over \(\mathcal{K}+B_2(0, s)\). For \(w \in \mathcal{K}, \tilde F_s(w)\) has following properties:
    \begin{enumerate}
    \item \(\widehat{F}(w) \leq \tilde F_s(w) \leq \widehat{F}(w)+L s\);
    \item \(\tilde F_s(w)\) is L-Lipschitz;
    \item \(\tilde F_s(w)\) is \(\frac{L \sqrt{d}}{r}\)-Smooth;
    \item For random variables \(v \sim \mathcal{U}_s\) and \(x\) uniformly from \(\{1,2,\dots, n\}\), one has
    \[
    \mathbb{E}[\nabla f(w+v, x)]=\nabla \tilde F_s(w)
    \]
    and
    \[
    \mathbb{E}\left[\left\|\nabla \tilde F_s(w)-\nabla f(w+v, x)\right\|_2^2\right] \leq L^2 .
    \]
    \end{enumerate}
\end{lemma}

\begin{definition}[Poisson Sampling for Convolutional Smoothing]\label{def:poisson_smooth}
    Since the smoother takes the form of an expectation, we can (independently) sample \(v_i \overset{\mathrm{iid}}{\sim} \mathcal{U}_s\), and compute \(\nabla f(w+v_i, x_i)\) for an estimate of the gradient of \(\tilde f_s(w, x_i)\). Similar calculation can be done for stochastic gradient.
    Let \(K\) denote the batch size (in expectation).
    With Poisson sampling of a rate \(p = K / n\),
    we compute an estimate of the stochastic gradient of the ERM smoother
    \begin{equation}\label{eq:poisson_smooth}
        \hat g = \frac{1}{K} \sum_{i=1}^n Z_i \nabla f(w+v_i, x_i),
    \end{equation}
    where \(Z_i \overset{\mathrm{iid}}{\sim} Bernoulli(p)\) and
    \(v_i \overset{\mathrm{iid}}{\sim} \mathcal{U}_s\). Each sample \(x_i\) is included in the sum independently with probability \(p\).
\end{definition}
    
Similar to the case for the regular stochstic gradient, we can obtain a \(O(L^2/m)\) bound for the variance of the estimate of the smoother, proved as follows.

\begin{theorem}\label{thm:poisson_smooth_var}
    With Poisson sampling of a rate \(p = K / n\), let \(\hat g\) denote the estimate of the above ERM smoother~\eqref{eq:poisson_smooth}. Then
    \(\hat g\) is an unbiased estimator of \(\nabla \hat F_{n_s}(w)\), and the variance of the estimate
     is \(4L^2/K\).
\end{theorem}
\begin{proof}
    By boundedness of \(\nabla f\), we can interchange the gradient and the expectation. We have \(\E[\nabla f(w+v_i, x_i)] = \nabla \E_{v_i \sim \mathcal{U}_s} f(w + v_i, x_i) = \tilde f_s\).
    Therefore, by linearity of expectation, we have
    \[
      \E [\hat g] = \frac{1}{K} \sum_{i=1}^n \E\left[Z_i \nabla f(w+v_i, x_i)\right] = \nabla \tilde F_s(w).
    \]
    The variance of the estimate is
    \begin{equation}
        \begin{aligned}
        V & := \mathbb{E} \| g - \nabla \tilde F_s(w) \|^2 \\
        & = \frac{1}{K^2} \E \left\|\sum_{i=1}^n \left(Z_i \nabla f(w+v_i, x_i) - p \nabla \tilde F_s(w)\right) \right\|^2.
        \end{aligned}
    \end{equation}
    Since \(\E \left[ Z_i \nabla f(w+v_i, x_i) - p \nabla \tilde F_s(w)\right] = 0\) and the samples \((Z_i, v_i)\) are independent, the cross terms in the expectation vanish. We have
    \begin{equation}\label{eq:variance_smooth}
        \begin{aligned}
        V & = \frac{1}{K^2} \sum_{i=1}^n \E \| Z_i \nabla f(w+v_i, x_i) - p \nabla \tilde F_s(w) \|^2 \\ %
        & = \frac{1}{K^2} \sum_{i=1}^n \left(
            p \E \| \nabla f(w+v_i, x_i) - p \nabla \tilde F_s(w) \|^2
            + (1-p) \E \|p \nabla \tilde F_s(w)\|^2
        \right) \\
        & \le \frac{1}{K^2} \sum_{i=1}^n \left(
            p (1+p)^2 L^2 + (1-p) p^2 L^2
        \right) \\
        & = \frac{1+3p}{K} \cdot L^2 \le \frac{4L^2}{K},
        \end{aligned}
    \end{equation}
    where the second line follows by conditioning on \(Z_i\), and the third line is due to the Lipschitz property of \(f\) and \(\tilde F_s\).
\end{proof}

\subsection{Algorithm and Analysis}

\begin{algorithm}[t]
    \caption{Convolutional Smoothing-Based
    Localized ISRL-DP Accelerated MB-SGD
    }\label{alg:phased_acc_nonsmooth_conv}
    \begin{algorithmic}[1]
    \Require
    Datasets \(X_l \in \mathcal{X}^n\) for \(l \in [N]\), loss function \(f\), constraint set \(\mathcal{W}\),
    initial point \(w_0\), subroutine parameters \(\{R_i\}_{i=1}^{\lfloor \log_2 n \rfloor} \subset \mathbb{N}\), \(\{K_i\}_{i=1}^{\lfloor \log_2 n \rfloor} \subset [n]\), smoothing parameter \(s\).
    \State Set \(\tau = \lfloor \log_2 n \rfloor\).
    \State Set \(p = \max(\tfrac{1}{2}\log_n(M) + 1, 3)\)
    \For{\(i = 1\) \textbf{to} \(\tau\)}
        \State Set \(\lambda_i = \lambda \cdot 2^{(i-1)p}\), \( n_i = \lfloor n / 2^i \rfloor\), \(D_i = 2L / \lambda_i\).
        \State Each silo \(l \in [N]\) draws disjoint batch \(B_{i,l}\) of \(n_i\) samples from \(X_l\).
        \State Let \(\hat F_i(w) = \frac{1}{n_i N} \sum_{l=1}^N \sum_{x_{l,j} \in B_{i,l}}  \tilde f_s(w; x_{l, j}) + \frac{\lambda_i}{2} \lVert w - w_{i-1} \rVert^2\), where \(\tilde f_s\) is the convolutional smoother of \(f\) with radius \(s\).
        \State Call the multi-stage \((\varepsilon, \delta)\)-ISRL-DP implementation of~\cref{alg:acc_mbsgd_nonsmooth}
        with loss function \(\hat F_i(w)\), data \(X_l = B_{i,l}\), \(R = R_i\), \(K = K_i\), initialization \(w_{i-1}\), constraint set \(\mathcal{W}_i = \{ w \in \mathcal{W} : \lVert w - w_{i-1} \rVert \leq D_i \}\), and \(\mu = \lambda_i\). Denote the output by \(w_i\).
    \EndFor
    \State \textbf{return} the last iterate \(w_{\tau}\)
\end{algorithmic}
\end{algorithm}
\begin{algorithm}[htb]
    \caption{Accelerated ISRL-DP MB-SGD for Convolutional Smoother}\label{alg:acc_mbsgd_nonsmooth}
    \begin{algorithmic}[1]
    \Require Datasets \(X_l \in \mathcal{X}^{n}\) for \(l \in [N]\), initial point \(w_0\), loss function \(\hat{F}(w) = \frac{1}{nN}\sum_{l=1}^N\sum_{x \in X_l} \tilde f_s(w,x) + \tfrac{\lambda}{2} \| w - w_0 \|^2\), constraint set \(\mathcal{W}\),
    strong convexity modulus \(\mu \geq 0\),
    privacy parameters \((\varepsilon, \delta)\),
    iteration count \(R \in \mathbb{N}\), (expected) batch size \(K \in [n]\), step size parameters \(\{\eta_r \}_{r \in [R]}, \{\alpha_r \}_{r \in [R]}\) specified in~\cref{app:multi-stage}.  
    \State  Initialize \(w_0^{ag} = w_0 \in \mathcal{W}\) and \(r = 1\).
    \For{\(r \in [R]\)}
    \State Server updates and broadcasts \\ \(w_r^{md} = \frac{(1- \alpha_r)(\mu + \eta_r)}{\eta_r + (1 - \alpha_r^2)\mu}w_{r-1}^{ag} + \frac{\alpha_r[(1-\alpha_r)\mu + \eta_r]}{\eta_r + (1-\alpha_r^2)\mu}w_{r-1}\)
    \For{\(i \in S_r\) \textbf{in parallel}}
    \State Let the Poisson sampling rate be \(p = K/n\).
    \State Silo \(i\) draws \(Z_{i,j} \overset{\mathrm{iid}}{\sim} Bernoulli(p), \, j \in [n]\) and
    \(v_{i,j} \overset{\mathrm{iid}}{\sim} \mathcal{U}_s, \, j \in [n]\)
    \State Sample privacy noise \(u_i \sim \mathcal{N}(0, \sigma^2 \mathbf{I}_d)\) for proper \(\sigma^2\).
    \State Silo \(i\) computes \(\widetilde{g}_r^{i} := \frac{1}{K} \sum_{j=1}^{n} Z_{i,j} \nabla f(w_r^{md} + v_{i,j}, x_{i,j}^r) + \lambda (w_r^{md} - w_0) + u_i\).
    \EndFor
    \State Server aggregates \(\widetilde{g}_{r} := \frac{1}{M} \sum_{i \in S_r} \widetilde{g}_r^{i}\)
    and updates:
    \State \(w_{r} := \argmin_{w \in \mathcal{W}}\left\{\alpha_r \left[\langle \widetilde{g}_{r}, w\rangle + \frac{\mu}{2}\|w_r^{md} - w\|^2 \right] \right. \)
    \(\left. + \left[(1-\alpha_r) \frac{\mu}{2} + \frac{\eta_r}{2}\right]\|w_{r-1} - w\|^2\right\}\).
    \State Server updates and broadcasts \\
    \(
    w_{r}^{ag} = \alpha_r w_r + (1-\alpha_r)w_{r-1}^{ag}.
    \)
    \EndFor \\
    \State \textbf{return:} \(w_R^{ag}\).
\end{algorithmic}
\end{algorithm}

For nonsmooth loss, we apply \cref{alg:phased_acc} to the convolutional smoother with radius \(s\), where \(s\) is to be determined.
We describe the algorithm in \cref{alg:phased_acc_nonsmooth_conv} and the subroutine in \cref{alg:acc_mbsgd_nonsmooth}. The changes compared to \cref{alg:phased_acc} are listed below.

In the main algorithm, \(f\) is replaced by the smoother \(f_s\). For the subroutine call in Line 7, we modify \cref{alg:acc_mbsgd} as follows. Let \(\frac{\lambda}{2} \| w - w_0 \|^2\) be the regularizer applied for the subroutine call.
\begin{itemize}
    \item We change the subsampling regime to Poisson sampling with rate \(p = K/n\) in line 5. For silo \(i\), we sample \(Z_{i,j} \overset{\mathrm{iid}}{\sim} \mathrm{Bernoulli}(p)\) for \(j \in [n]\) and include \(x_{i,j}\) in the sum with probability \(p\). DP noise is sampled \(u_i \sim \mathcal{N} (0, \sigma^2 I_d)\) as before.
    \item We change the gradient computation in Line 6. We sample \(v_{i,j} \overset{\mathrm{iid}}{\sim} \mathcal{U}_s\) and compute the estimate of the gradient of the smoother \(\frac{1}{pn} \sum_{j=1}^n Z_{i,j} \nabla \tilde f_s(w+v_{i,j}, x_j) + \lambda (w - w_0) + u_i\).
\end{itemize}

For the analysis of the algorithm, we note that using Poisson sampling does not change the privacy analysis (up to some constant factors). 
We can apply the same proof in \cref{proof:main} to the smoother \(f_s\) except for one minor change in~\eqref{eq:Fi_excess_risk}, since \cref{thm:lowy_erm_bound} does not directly apply as we have changed the gradient estimator computation. However, by \cref{thm:poisson_smooth_var}, the variance of the estimate of the smoother is still \(O\left(\frac{L^2}{m}\right)\) without considering the DP noise.
Therefore, the conclusion of \cref{thm:lowy_erm_bound} still holds for our smoother.

By the first property in \cref{lem:conv_smooth_properties} that relates \(\tilde F_s\) to \(\hat F\). It suffices to choose \(s\) to match the difference \(Ls\) with the excess risk bound. We have the following result for nonsmooth loss via convolutional smoothing.

\begin{theorem}[Nonsmooth FL via convolutional smoothing]\label{thm:main-conv_nonsmooth}
    Assume only that \(f(\cdot, x)\) is \(L\)-Lipschitz and convex for all \(x \in \mathcal{X}\).
    Let \(\varepsilon \leq 2  \ln(2/\delta)\), \(\delta \in (0,1)\). Choose the following convolutional smoothing parameter:
    \[s = \frac{D}{\sqrt{M}} \left(
        \frac{1}{\sqrt{n}} + \frac{\sqrt{d \log(1 / \delta)}}{\varepsilon n}
        \right).
    \]
    Let \(\beta\) = \(L \sqrt{d} / s\) and
    choose \(R_i \approx \max \left(\sqrt{\frac{\beta + \lambda_i}{\lambda_i}} \ln\left(\frac{\Delta_i \lambda_i M \varepsilon^2 n_i^2}{L^2 d}\right),
    \frac{\varepsilon^2 n_i^2}{K_i d \ln(1/\delta)}\right)\), where \(LD \ge \Delta_i \ge \hat F_i (w_{i-1}) - \hat F_i(\hat w_i)\), \(K_i \geq \frac{\varepsilon n_i}{4 \sqrt{2R_i \ln(2/\delta)}}\), \(\sigma_i^2 = \frac{256 L^2 R_i \ln(\frac{2.5 R_i}{\delta}) \ln(2/\delta)}{n_i^2 \varepsilon^2}\), and \begin{equation}\label{eq:lambda1}
        \lambda = \frac{L}{D n\sqrt{M}} \max \left\{\sqrt{n}, \frac{\sqrt{d \ln (1/\delta)}}{\varepsilon}\right\}.
    \end{equation}
    Then, the output of \cref{alg:phased_acc_nonsmooth_conv} is \((\varepsilon, \delta)\)-ISRL-DP and achieves the following excess risk bound:
    \begin{equation}\label{eq:excess_risk1}
        \E F(w_{\tau}) - F(w^*) = \widetilde{O}\left( \frac{LD}{\sqrt{M}} \left(
            \frac{1}{\sqrt{n}} + \frac{\sqrt{d \log(1 / \delta)}}{\varepsilon n}
            \right)
            \right).
    \end{equation}
    The communication complexity (in expectation) is
    \begin{equation}\label{eq:smooth_comm1}
        \widetilde O\left(\max\left\{1, 
            d^{1/4}\sqrt{M} \min \left\{\sqrt{n}, \frac{\varepsilon n}{\sqrt{d \ln (1/\delta)}}\right\}, \frac{\varepsilon^2 n}{d \ln(1/\delta)}
        \right\}\right),
    \end{equation}
    when \(K_i = n_i\).
    When \(d = \Theta(n)\), and \(\varepsilon = \Theta(1)\), the gradient complexity is
    \begin{equation*}
        \widetilde{O}\left( M^{3/2} n^{3/4} + M^{5/4} n^{11/8}\right).
    \end{equation*}
\end{theorem}

We present the complete proof below.
By properties of the smoother (\ref{lem:conv_smooth_properties}), we know that \(\tilde f_s\) is convex, \(L\)-Lipschitz and \(\beta\)-smooth. We can thus reuse \cref{lem:Fhat_property}, restated below. 
\begin{lemma}
Let \(\hat{w}_i=\argmin_{w \in \mathcal{W}} \hat F_i(w)\).
We have \(\hat{w}_i \in \mathcal{W}_i\) and \(\hat F_i\) is \(3L\)-Lipschitz and \((\beta + \lambda_i)\)-smooth.
\end{lemma}

2. We have the following bounds (same as \cref{lem:iter_risk}) that relate the private solution \(w_i\) and the true solution \(\hat w_i\) of \(\hat F_i\).
\begin{lemma}
    In each phase \(i\), the following bounds hold:
    \begin{equation}\label{eq:Fi_excess_risk1}
        \E [\hat F_i(w_i) - \hat F_i(\hat w_i)] = \widetilde{O}\left(\frac{L^2}{\lambda_i M}\cdot \frac{d \ln(1/\delta)}{\varepsilon^2 n_i^2} \right),
    \end{equation}
    \begin{equation}\label{eq:w_dist1}
        \E\left[\left\|w_i-\hat{w}_i\right\|^2\right] \le \widetilde O\left(\frac{L^2}{\lambda_i^2 M} \cdot \frac{d \log(1 / \delta)}{n_i^2 \varepsilon^2}\right).
    \end{equation}
\end{lemma}
To prove this, it suffices to show that the results \cref{thm:lowy_erm_bound} hold for the multi-stage implementation of \cref{alg:acc_mbsgd_nonsmooth}.
The proof is essentially the same as that in~\citet{lowy2023private} by observing that the variance of the estimate is still \(O(L^2/m)\) by \cref{thm:poisson_smooth_var}.

For simplicity and consistency with the proof in \cref{proof:main}, we will redefine \(F\) as the smoother \(\tilde F_{s}\) in the following and use \(F_0\) to denote the original loss function.
We will apply the following bound for the multi-stage implementation.
\begin{lemma}\label{lem:multi_stage_bound}
    Let \(f: \mathcal{W} \to \mathbb{R}^d\) be \(\mu\)-strongly convex and \(\beta\)-smooth, and suppose that the unbiased stochastic gradients \(\widetilde{g}(w_r)\) at each iteration \(r\) have bounded variance \(\mathbb{E}\|\widetilde{g}(w_r) - \nabla F(w_t) \|^2 \leq V^2\). If \(\widehat{w}_R^{ag}\) is computed by \(R\) steps of the Multi-Stage Accelerated MB-SGD, then \[
    \mathbb{E}F(\widehat{w}_R^{ag}) - F^{*}\lesssim \Delta \exp\left(-\sqrt{\frac{\mu}{\beta}} R\right) + \frac{V^2}{\mu R},
    \]
    where \(\Delta = F(w_0) - F^{*}\).
\end{lemma}

\begin{proof}
By \cref{def:poisson_smooth}, we know that our gradient estimator in \cref{alg:acc_mbsgd_nonsmooth}
\[
    \tilde g_r = \lambda (w_r^{md} - w_0) + \frac{1}{KM} \sum_{i \in S_r} \sum_{j=1}^{n} Z_{i,j} \nabla f(w_r^{md} + v_{i,j}, x_{i,j}^r) + \frac{1}{M} \sum_{i \in S_r} u_i,
\]
is an unbiased estimator of \(\nabla \hat F(w_r^{md})\), and the variance of the estimate has the bound
\(
    V := \frac{4L^2}{KM} + \frac{d \sigma^2}{M},
\)
where the first term is the variance bound in \cref{thm:poisson_smooth_var}, and the second term is due to the DP noise.

Applying \cref{lem:multi_stage_bound} to \(\hat F_i\), with \(\beta\) replaced by \(\beta + \lambda_i\), \(\mu\) set to \(\lambda_i\), we get the following the excess risk bound
\[
    \E [\hat F_i(w_i) - \hat F_i(\hat w_i)] \lesssim \Delta_i \exp\left(-\sqrt{\frac{\lambda_i}{\beta + \lambda_i}}R_i \right) + \frac{L^2}{\lambda_i}\left(\frac{1}{MKR} + \frac{d \ln^2(R_i \delta)}{M \varepsilon^2 n^2}\right).
\]
The excess risk bound~\eqref{eq:Fi_excess_risk1} follows by
plugging in our choice of \(R_i\).
The bound~\eqref{eq:w_dist1} then follows from the inequality below due to \(\lambda_i\)-strong convexity,
\begin{equation*}
    \frac{\lambda_i}{2} \E\left[\left\|w_i-\hat{w}_i\right\|^2\right] \le \E [\hat F_i(w_i) - \hat F_i(\hat w_i)].
\end{equation*}
\end{proof}

3. As a consequence, we can reuse \cref{lem:excess_risk_wi}, restated below.
\begin{lemma}
    Let \(w \in \mathcal{W}\). We have
    \begin{equation*}
        \E[F(\hat w_i)] - F(w) \le \frac{\lambda_i \E[\| w - w_{i-1}\|^2]}{2} + \frac{4 \cdot (3L)^2}{\lambda_i n_i M}.
    \end{equation*}
\end{lemma}

Putting these results together, we now prove the theorem.
\begin{proof}[Proof]
We note that the proof is essentially the same as that in \cref{proof:main}.

\textbf{Privacy.}
Note that using Poisson sampling does not change the privacy analysis (up to some constant factors).
By our choice of \(\sigma_i\), each phase of the algorithm is \((\varepsilon, \delta)\)-ISRL-DP.\@ Since the batches \(\{B_{i,l}\}_{i=1}^{\tau}\) are disjoint for all \(l \in [M]\), the privacy guarantee of the entire algorithm follows from parallel composition.

\textbf{Excess risk.}
Given the lemmas introduced above, by the identical reasoning as in \cref{proof:main}, we have the following excess risk bound for the smoother,
\[
    \E F(w_{\tau}) - F(w^*) \le \widetilde{O}\left( \frac{LD}{\sqrt{M}} \left(
    \frac{1}{\sqrt{n}} + \frac{\sqrt{d \log(1 / \delta)}}{\varepsilon n}
    \right)
    \right).
\]
By \cref{lem:conv_smooth_properties}, we can relate this to the true loss \(F_0\) as follows
\[
    \E F_0(w_{\tau}) - F_0(w^*) \le \E F(w_{\tau}) - F(w^*) + Ls.
\]
The bound then follows from our choice of \(s\).

\textbf{Complexity.}
When we use the full batch in each round, that is, when \(K_i = n_i\),
hiding logarithmic factors, communication complexity is 
\begin{equation}
    \begin{aligned}
        \sum_{i=1}^{\tau} R_i
        & = \sum_{i=1}^{\tau} \max\left\{1, \sqrt{\frac{\beta + \lambda_i}{\lambda_i}} \ln\left(\frac{\Delta \lambda_i M \varepsilon^2 n_i^2}{L^2 d}\right), \frac{\varepsilon^2 n_i^2}{n_i d \ln(1/\delta)}
        \right\} \\
        & = \widetilde O \left(\max\left\{\lfloor \log_2 n \rfloor, \sqrt{\frac{\beta}{\lambda}}, \frac{\varepsilon^2 n}{d \ln(1/\delta)}
        \right\} \right)\\
        & = \widetilde O\left(\max\left\{1, 
            d^{1/4}\sqrt{M} \min \left\{\sqrt{n}, \frac{\varepsilon n}{\sqrt{d \ln (1/\delta)}}\right\}, \frac{\varepsilon^2 n}{d \ln(1/\delta)}
        \right\}\right),
    \end{aligned}
\end{equation}
where in the last step is due to our choice of 
\(\beta = L \sqrt{d} / s\) and \(\lambda\).

For gradient complexity, we can follow the same analysis in \cref{sec:grad_complex} except that due to Poisson sampling our gradient complexity will be in expectation, instead of deterministic. We only consider the case when \(d = \Theta(n)\) and \(\varepsilon = \Theta(1)\) here to make it simple. Keep terms involving \(M\), \(n\) only and drop \(\widetilde O\) for simplicity.
We have \(s = \frac{D}{\sqrt{nM}}\), \(\beta = L \sqrt{d} / s = \frac{nL\sqrt{M}}{D}\).

Plugging \(\beta\) into \cref{eq:grad_cost_d_eq_n},
we obtain the gradient complexity (in expectation) of \(\widetilde{O}\left( M^{3/2} n^{3/4} + M^{5/4} n^{11/8}\right)\).

\end{proof}

\section{Precise Statement and Proof of~\cref{thm:smoothing}}\label{app:smoothing}

\begin{lemma}\citep{nesterov2005smooth}\label{lem:nesterov_smooth}
Let \(f: \mathcal{W} \to \mathbb{R}^d\) be convex and \(L\)-Lipschitz and \(\beta > 0\). Then the \(\beta\)-Moreau envelope \(f_{\beta}(w):= \min_{v \in \mathcal{W}} \left(f(v) + \frac{\beta}{2}\|w - v\|^2 \right)\) satisfies: \\
(1) \(f_{\beta}\) is convex, \(2L\)-Lipschitz, and \(\beta\)-smooth;\\
(2) \(\forall w \in \mathcal{W}\), \(f_{\beta}(w) \leq f(w) \leq f_{\beta}(w) + \frac{L^2}{2 \beta}\); \\
(3) \(\forall w \in \mathcal{W}\), \(\nabla f_{\beta}(w) = \beta(w - \mathrm{prox}_{f/\beta}(w))\); \\
where the prox operator of \(f: \mathcal{W} \to \mathbbm{R}^d\)  is defined as 
\(\mathrm{prox}_f(w):= \argmin_{v \in \mathcal{W}} \left(f(v) + \frac{1}{2} \norm{w - v}^2\right) \).
\end{lemma}

Our algorithm for nonsmooth functions uses Lemma~\ref{lem:nesterov_smooth} and \cref{alg:phased_acc} as follows: First, property (1) above allows us to apply \cref{alg:phased_acc} to optimize \(f_{\beta}\) and obtain an excess population risk bound for \(f_{\beta}\), via \cref{thm:main}. Inside \cref{alg:phased_acc}, we will use property (3) to compute the gradient of \(f_{\beta}\). Then, property (2) enables us to extend the excess risk guarantee to the true function \(f\), for a proper choice of \(\beta\). 

\begin{theorem}[Precise Statement of \cref{thm:smoothing}]
    Let \(\varepsilon \leq 2  \ln(2/\delta)\), \(\delta \in (0,1)\). Choose \(R_i \approx \max \left(\sqrt{\frac{\beta + \lambda_i}{\lambda_i}} \ln\left(\frac{\Delta_i \lambda_i M \varepsilon^2 n_i^2}{L^2 d}\right),
\mathbbm{1}_{\{M K_i  < Nn_i\}} \frac{\varepsilon^2 n_i^2}{K_i d \ln(1/\delta)}\right)\), where \(LD \ge \Delta_i \ge \hat F_i (w_{i-1}) - \hat F_i(\hat w_i)\), \(K_i \geq \frac{\varepsilon n_i}{4 \sqrt{2R_i \ln(2/\delta)}}\), \(\sigma_i^2 = \frac{256 L^2 R_i \ln(\frac{2.5 R_i}{\delta}) \ln(2/\delta)}{n_i^2 \varepsilon^2}\), and
\begin{equation*}
    \lambda = \frac{L}{D n\sqrt{M}} \max \left\{\sqrt{n}, \frac{\sqrt{d \ln (1/\delta)}}{\varepsilon}\right\}.
\end{equation*}
There exist choices of \(\beta\) such that running \cref{alg:phased_acc} with \(f_{\beta}(w, x):= \min_{v \in \mathcal{W}} \left(f(v, x) + \frac{\beta}{2}\|w - v\|^2 \right)\) yields an \((\varepsilon,\delta)\)-ISRL-DP algorithm with optimal excess population risk as in~\eqref{eq:excess_risk}. The communication complexity is 
    \begin{equation*}
         \widetilde O\left(\max \left\{1, 
                    \sqrt{M}
                    \min \left\{\sqrt{n}, \frac{\varepsilon n}{\sqrt{d \ln (1/\delta)}}\right\}
                    ,
                 \mathbbm{1}_{\{M < N\}} \frac{\varepsilon^2 n}{d \ln(1/\delta)}
                 \right\}
            \right).
    \end{equation*}
\end{theorem}

\begin{proof}
\textbf{Privacy.} For any given \(\beta\),  the privacy guarantee is immediate since we have already showed that \cref{alg:phased_acc} is \((\varepsilon,\delta)\)-ISRL-DP.\@

\textbf{Excess risk.}
By \cref{lem:nesterov_smooth} part 2, we have
\begin{equation*}
    \E F(w_{\tau}) - F(w^*) \le \E F(w_{\tau}) - F(w^*) + \frac{L^2}{2\beta}.
\end{equation*}
It suffices to choose \(\beta = \frac{L\sqrt{M}}{D} \min \left\{
 \sqrt{n}, \frac{\varepsilon n}{\sqrt{d \log(1 / \delta)}}
\right\}
\).

\textbf{Communication complexity.}
The result follows by plugging in our choice of \(\beta\) into~\eqref{eq:smooth_comm}.

\end{proof}

\section{More Details on the Communication Complexity Lower Bound (\cref{thm:lower-com})}\label{app: communication lower bound}

We begin with a definition~\citep{woodworth2020minibatch}:
\begin{definition}[Distributed Zero-Respecting Algorithm]
For \(v \in \mathbb{R}^d\), let \(\text{supp}(v):= \{j \in [d] : v_j \neq 0\}\). Denote the random seed of silo \(m\)'s gradient oracle in round \(t\) by \(z_t^m\). An optimization algorithm is \textit{distributed zero-respecting with respect to \(f\)} if for all \(t \geq 1\) and \(m \in [N]\), the \(t\)-th query on silo \(m\), \(w_t^{m}\) satisfies \[
\text{supp}(w_t^m) \subseteq \bigcup_{s < t} \text{supp}(\nabla f(w_s^m, z_s^m)) \cup \bigcup_{m' \neq m} \bigcup_{s \leq \pi_m(t, m')} \text{supp}(\nabla f(w_s^{m'}, z_s^{m'})),
\]
where \(\pi_m(t, m')\) is the most recent time before \(t\) when silos \(m\) and \(m'\) communicated with each other. 
\end{definition}

\begin{theorem}[Re-statement of~\cref{thm:lower-com}]
Fix \(M=N\) and suppose \(\mathcal{A}\) is a distributed zero-respecting algorithm with excess risk \[
\E F(\mathcal{A}(X)) - F^* \lesssim \frac{LD}{\sqrt{N}}\left(\frac{1}{\sqrt{n}} + \frac{\sqrt{d \ln(1/\delta)}}{\varepsilon n}\right)\]
in \(\leq R\) rounds of communications on any \(\beta\)-smooth FL problem with heterogeneity \(\zeta_*\) and \(\|w_0 - w^*\| \leq D\). Then,
\begin{align*}
R &\gtrsim N^{1/4}\left(\min \left\{\sqrt{n}, \frac{\varepsilon n}{\sqrt{d \ln (1/\delta)}}\right\}\right)^{1/2} \times \min\left(\frac{\sqrt{\beta D}}{\sqrt{L}}, \frac{\zeta_*}{\sqrt{\beta L D}}\right).
\end{align*}
\end{theorem}
\begin{proof}
Denote the worst-case excess risk of \(\mathcal{A}\) by \[
\alpha := \E F(\alg(X)) - F^*.
\]
Now, \citet[Theorem 4]{woodworth2020minibatch} implies that any zero-respecting algorithm has \begin{align*}
    R \gtrsim \min\left\{\frac{\zeta_*}{\sqrt{\beta \alpha}}, \frac{D \sqrt{\beta}}{\sqrt{\alpha}}\right\}. 
\end{align*}

Plugging in \(\alpha = \frac{LD}{\sqrt{N}}\left(\frac{1}{\sqrt{n}} + \frac{\sqrt{d \ln(1/\delta)}}{\varepsilon n}\right)\) proves the lower bound. 
\end{proof}

\section{Precise Statement and Proof of \cref{thm:nonsmooth}}
\begin{theorem} [Precise statement of \cref{thm:nonsmooth}]\label{thm:nonsmooth-app}
    Choose
    \begin{equation}\label{eq:eta}
        \eta = \frac{D\sqrt{M}}{L} \min \left\{\frac{1}{\sqrt{n}}, \frac{\varepsilon}{\sqrt{d \ln (1/\delta)}}\right\},
    \end{equation}
\(K_i \ge \max \left(1, \frac{\varepsilon n_i}{4 \sqrt{2R_i \ln(2/\delta)}}\right)\), \(R_i = \min \left(
    M n_i, \frac{M\varepsilon^2 n_i^2}{d}
    \right) + 1\), and \(\sigma_i^2 = \frac{256 L^2 R_i \ln(\frac{2.5 R}{\delta}) \ln(2/\delta)}{n_i^2 \varepsilon^2}\) for \(i \in [k]\)in~\cref{alg:phased_acc_nonsmooth}. Then, \cref{alg:phased_acc_nonsmooth} is \((\varepsilon, \delta)\)-ISRL-DP and achieves the following excess risk bound:
    \begin{equation*}
        \E F(w_{\tau}) - F(w^*) = \widetilde{O}\left( \frac{LD}{\sqrt{M}} \left(
            \frac{1}{\sqrt{n}} + \frac{\sqrt{d \log(1 / \delta)}}{\varepsilon n}
            \right)
            \right).
    \end{equation*}
    Further, the communication complexity is
    \begin{equation*}
        \sum_{i=1}^{\tau} R_i = \widetilde{O} \left( \min \left(
        nM, \frac{M\varepsilon^2 n^2}{d}
    \right) + 1\right),
    \end{equation*}
    and the subgradient complexity is
    \begin{equation*}
        \sum_{i=1}^{\tau} R_i K_i \cdot M = \widetilde{O} \left(
            M + M^2
                \min \left(
                    n, \frac{\varepsilon^2 n^2}{d}
                \right) + \varepsilon Mn + M^{3/2}
                \min \left(
                    \varepsilon n^{3/2}, \frac{\varepsilon^2 n^2}{\sqrt{d}}
                \right)
        \right).
\end{equation*}
\end{theorem}

We need the following result in convex optimization to bound the excess risk in each call of ISRL-DP MB-Subgradient Method.
\begin{lemma}\citep[Theorem 6.2]{bubeck2015convex}\label{lem:bubeck}
    Let \(g\) be \(\lambda\)-strongly convex, and assume that the stochastic subgradient oracle returns a stochastic subgradient \(\wt{g}(w)\) such that \(\mathbb{E}\wt{g}(w)\in \partial g(w)\) and  \(\mathbb{E}\|\widetilde{g}(w)\|_2^2 \leq B^2\). Then, the stochastic subgradient method \(w_{r+1} = w_r - \gamma_r \wt{g}(w_r)\) with \(\gamma_r= \frac{2}{\lambda(r+1)}\) satisfies
    \begin{equation}
        \mathbb{E} g\left(\sum_{r=1}^R \frac{2 r}{R(R+1)} w_r\right)-g\left(w^*\right) \leq \frac{2 B^2}{\lambda(R+1)} .
    \end{equation}
\end{lemma}

As a consequence, we have the following result:
\begin{lemma}\label{lem:convex_sgd_risk}
    In each phase \(i\) of~\cref{alg:phased_acc_nonsmooth}, the following bounds hold:
    \begin{equation}\label{eq:sgd_Fi_excess_risk}
        \E [\hat F_i(w_i) - \hat F_i(\hat w_i)] = \widetilde{O}\left(
            \frac{L^2 \eta_i}{M} + \frac{dL^2 \eta_i}{M n_i \varepsilon^2}
        \right),
    \end{equation}
    \begin{equation}\label{eq:sgd_w_dist}
    \E\left[\left\|w_i-\hat{w}_i\right\|^2\right] \le
    \widetilde{O}\left(
            \frac{L^2 \eta_i^2 n_i}{M} + \frac{dL^2 \eta_i^2}{M \varepsilon^2 }
    \right).
    \end{equation}
\end{lemma}

\begin{proof}
Note that the second moment of each noisy aggregated subgradient in round \(r\) of~\cref{alg:MB-SGD} is bounded by
\begin{equation*}
    \E\left\|\tilde g_r\right\|^2
    = \left\| u_r + \frac{1}{M_r} \sum_{i \in S_r} \frac{1}{K}\sum_{j=1}^{K} \nabla f(w_r, x_{i,j}^r)\right\|^2 \le 2 L^2 + \frac{2d \sigma^2}{M}.
\end{equation*}
    
Given our choice of step sizes in \cref{alg:phased_acc_nonsmooth}, we can verify the assumptions in \cref{lem:bubeck} are met for \(\hat F_i\). Recall \(w_i\) is the output of each phase and \(\hat{w}_i=\argmin_{w \in \mathcal{W}} \hat F_i(w)\).
It follows from \cref{lem:bubeck} that,
\begin{equation*}
    \begin{aligned}
    \E [\hat F_i(w_i) - \hat F_i(\hat w_i)]
    & \le \frac{2}{\lambda_i (R_i+1)} \cdot \left(2 L^2 + \frac{2d \sigma_i^2}{M}\right) \\
    & = \widetilde{O}\left(
            \frac{L^2 \eta_i n_i}{ R_i} + \frac{dL^2 \eta_i}{M n_i \varepsilon^2}
        \right) \\
    & = \widetilde{O}\left(
    \frac{L^2 \eta_i}{M} + \frac{dL^2 \eta_i}{M n_i \varepsilon^2}
\right),
    \end{aligned}
\end{equation*}
as desired~\eqref{eq:sgd_Fi_excess_risk}, where the last step is due to our choice of \(R_i = \min \left(
    M n_i, \frac{M\varepsilon^2 n_i^2}{d}
\right) + 1\).

Using the \(\lambda_i\)-strong convexity,
we have
\begin{equation*}
    \frac{\lambda_i}{2} \E\left[\left\|w_i-\hat{w}_i\right\|^2\right] \le \E [\hat F_i(w_i) - \hat F_i(\hat w_i)].
\end{equation*}
The bound~\eqref{eq:sgd_w_dist} follows.
\end{proof}

Now we are ready to prove the theorem.
\begin{proof}[Proof of \cref{thm:nonsmooth-app}]
\textbf{Privacy.}
The privacy analysis of~\citet[Theorem D.1]{lowy2023private} for ISRL-DP MB-SGD holds verbatim in the nonsmooth case when we replace subgradients by gradients: the only property of \(f(\cdot, x)\) that is used in the proof is Lipschitz continuity. Thus, \cref{alg:MB-SGD} is \((\varepsilon, \delta)\)-ISRL-DP if the noise variance is \(\sigma^2 \geq \frac{256 L^2 R \ln(2.5 R/\delta) \ln(2/\delta)}{\varepsilon^2 n^2}\). By our choice of \(\sigma_i^2\), we see that phase \(i\) of~\cref{alg:phased_acc_nonsmooth} is \((\varepsilon, \delta)\)-ISRL-DP on data \(\{B_{i,l}\}_{l=1}^N\). Since the batches \(\{B_{i,l}\}_{i=1}^{\tau}\) are disjoint for all \(l \in [N]\), the full~\cref{alg:phased_acc_nonsmooth} is \((\varepsilon, \delta)\)-ISRL-DP by parallel composition~\citep{mcsherry2009privacy}. 

\textbf{Excess Risk.}
Define \(\hat w_0 = w^*\) and write
\begin{equation*}
    \E F(w_{\tau}) - F(w^*) = \E[F(w_{\tau}) - F(\hat{w}_{\tau})] + \sum_{i=1}^{\tau} \E[F(\hat{w}_i) - F(\hat{w}_{i-1})].
\end{equation*}
Since \(\tau = \lfloor \log_2 n \rfloor\), we have \(n_{\tau} = \Theta(1)\) and \( \eta_{\tau} = \eta \Theta(n^{-p})\).
By~\eqref{eq:sgd_w_dist} and Jensen's Inequality \(\E Z \le \sqrt{\E Z^2} \), we bound the first term as follows
\begin{equation*}
    \begin{aligned}
        \E[F(w_{\tau}) - F(\hat{w}_{\tau})] \le L \E \left[\|w_{\tau} - \hat{w}_{\tau}\|\right] &
        \le \widetilde{O} \left(
            \frac{L^2 \eta_{\tau} \sqrt{n_{\tau}}}{\sqrt{M}} + \frac{\sqrt{d} L^2 \eta_{\tau}}{\sqrt{M} \varepsilon}
        \right) \\
        & \le \widetilde{O} \left(
            \frac{L^2 \eta}{n^p \sqrt{M}} + \frac{\sqrt{d} L^2 \eta}{\sqrt{M} n^p \varepsilon}
        \right)  \\
        & \le \widetilde{O} \left(
            \frac{LD}{n^{p-\tfrac{1}{2}}}
        \right) \le \widetilde{O} \left(
            \frac{LD}{\sqrt{nM}}
        \right),
    \end{aligned}
\end{equation*}
where the last two steps are due to the choice of \(\eta\) per~\eqref{eq:eta} and~\eqref{eq:p}.

By~\eqref{eq:sgd_Fi_excess_risk} and using the fact that \(1/n_i\) and \(\eta_i\) decrease geometrically, we have
\begin{equation*}
    \begin{aligned}
        \sum_{i=1}^{\tau} \E[F(\hat{w}_i) - F(\hat{w}_{i-1})]
        & \le \sum_{i=1}^{\tau} \left(\frac{\lambda_i \E[\| w_{i-1} - \hat w_{i-1}\|^2]}{2} + \frac{4 \cdot (3L)^2}{\lambda_i n_i M}\right) \\
        & \le \widetilde O\left( \frac{D^2}{\eta n} +
            \sum_{i=2}^{\tau}
                \frac{L^2 \eta_i}{M} + \frac{dL^2 \eta_i}{M n_i \varepsilon^2 } +\sum_{i=1}^{\tau}\frac{L^2 \eta_i}{M}
        \right) \\
        & \le \widetilde O\left( \frac{D^2}{\eta n} +
        \frac{L^2 \eta}{M} + \frac{L^2 \eta d}{M n \varepsilon^2}
        \right).
    \end{aligned}
\end{equation*}
Plugging in the choice of \(\eta\) per~\eqref{eq:eta} gives the desired excess risk.

\textbf{Communication complexity.}
Summing geometric series, we obtain the communication complexity as follows
\begin{equation*}
    \sum_{i=1}^{\tau} R_i = \widetilde{O} \left( \min \left(
        nM, \frac{M\varepsilon^2 n^2}{d}
    \right) \right) + 1.
\end{equation*}
\textbf{Subgradient complexity.}
Recall \(K_i \ge \max \left(1, \frac{\varepsilon n_i}{4 \sqrt{2R_i \ln(2/\delta)}}\right)\). Choosing the minimum \(K_i\),
we have
\begin{equation*}
    \begin{aligned}
        \sum_{i=1}^{\tau} R_i K_i \cdot M & = \widetilde{O} \left(
            \max \left(RM, \varepsilon nM \sqrt{R}\right)
        \right) \\
        & = \widetilde{O} \left(
            M + M^2
                \min \left(
                    n, \frac{\varepsilon^2 n^2}{d}
                \right) + \varepsilon n M + M^{3/2}
                \min \left(
                    \varepsilon n^{3/2}, \frac{\varepsilon^2 n^2}{\sqrt{d}}
                \right)
        \right).
    \end{aligned}
\end{equation*}
\end{proof}
\section{A Stability Result}\label{proof:stability}
In order to bound the excess risk of the function \(F_i\), we require the following generalization of Theorem 6 from~\citet{shalev-shwartzStochasticConvexOptimization} which provides a stability result.
\begin{lemma}\label{lem:stab}
    Let \(g(w,x), \, w \in \mathcal{W}\) be \(\lambda\)-strongly convex and \(L\)-Lipschitz in \(w\) for all \(x \in \mathcal{X}\).
    Let \(X = (x_1, x_2, \dots, x_m)\) be a set of \(m\) independent samples such that \(x_i\) is sampled from their corresponding distribution \(\mathcal{D}_i\) for \(i \in [m]\). We write \(X \sim \mathcal{D}\) for short.

    Let \(\hat G(w) = \frac{1}{m} \sum_{i=1}^m g(w, x_i)\) and let \(\hat w = \argmin_{w \in \mathcal{W}} \hat G(w)\) be the empirical minimizer. Let \(G(w) = \E_{X \sim \mathcal{D}} \left[\frac{1}{m} \sum_{i=1}^m g(w, x_i)\right]\) and let \(w^* = \argmin G(w)\) be the population minimizer.
    Then, for any \(w \in \mathcal{W}\), we have
\begin{equation*}
    \E[G(\hat w)] - G(w) \le \frac{4L^2}{\lambda m}.
\end{equation*}
\end{lemma}
\begin{proof}
    We will use a stability argument. Let \(X' = (x'_1, x'_2, \dots, x'_n)\) be a set of \(m\) independent samples from \(\mathcal{D}\), and let \(X^{(i)} = (x_1, \dots, x_{i-1}, x'_i, x_{i+1}, \dots, x_m)\) be the dataset with the \(i\)-th element replaced by \(x'_i\).
    
    Let \(\hat G^{(i)}(w) = \frac{1}{m} \left( f(w, x_i') + \sum_{j \ne i} f(w, x_j) \right)\) and \(\hat w^{(i)} = \min \hat G^{(i)}(w)\) be its minimizer. We have
    \begin{equation*}
        \begin{aligned}
            \hat G(w^{(i)}) - \hat G(\hat w)
            & = \frac{g(\hat{w}^{(i)}, z_i)-g(\hat{w}, z_i)}{m}+\frac{\sum_{j \neq i}\left(g(\hat{w}^{(i)}, z_i)-g(\hat{w}, z_i)\right)}{m} \\
            & = \frac{g(\hat{w}^{(i)}, z_i)-g(\hat{w}, z_i)}{m}+\frac{g(\hat{w}, z_i^{\prime})-g(\hat{w}^{(i)}, z_i^{\prime})}{m} \\
            & +\left(\hat{G}^{(i)}(\hat{w}^{(i)})-\hat{G}^{(i)}(\hat{w})\right) \\
            & \le \frac{\left|g(\hat{w}^{(i)}, z_i)-g(\hat{w}, z_i)\right|}{m}+\frac{\left|g(\hat{w}, z_i^{\prime})-g(\hat{w}^{(i)}, z_i^{\prime})\right|}{m} \\
            & \le \frac{2 L}{m}\left\|\hat{w}^{(i)}-\hat{w}\right\|,
            \end{aligned}
    \end{equation*}
where the first inequality follows from the definition of the minimizer \(w^{(i)}\), and the second inequality follows from the Lipschitzness of \(g\).

By strong convexity of \(\hat G\), we have
\begin{equation*}
    \hat G(\hat w^{(i)}) \ge \hat G(\hat w) + \frac{\lambda}{2} \left\|\hat w^{(i)} - \hat w\right\|^2. 
\end{equation*}
Therefore, we have \( \left\|\hat w^{(i)} - \hat w\right\| \le 4L / (\lambda m) \). Thus, ERM satisfies \(4L/\lambda m\) uniform argument stability, and (by Lipschitz continuity of \(g(\cdot, x)\)) \(4L^2/\lambda m\) uniform stability. 

Next, we show that this stability bound implies the desired generalization error bound (even if the samples are not drawn from an identical underlying distribution). 
Note that \(X\) and \(X'\) are independently sampled from \(\mathcal{D}\). By symmetry (renaming \(x_i\) to \(x_i'\)), we know that
\(\E_{X, X'} [g(\hat w, x_i)] = \E_{X, X'} [g(\hat w^{(i)}, x_i')]\).
Therefore, we have
\begin{equation*}
    \begin{aligned}
        \E [\hat G(\hat w)]
        & = \E_{X} \left[ \frac{1}{m} \sum_{i=1}^n g(\hat w, x_i) \right] \\
        & = \E_{X, X'} \left[ \frac{1}{m} \sum_{i=1}^n g(\hat w^{(i)}, x_i') \right] \\
        & = \E_{X, X'} \left[ \frac{1}{m} \sum_{i=1}^n g(\hat w, x_i') + \delta \right] \\
        & = \E [G(\hat w)] + \delta.
    \end{aligned}
\end{equation*}
By Lipschitz continuity of \(f\), we have
\begin{equation*}
    \begin{aligned}
    \delta
    &= \E_{X, X'} \frac{1}{m} \sum_{i=1}^n \left( g(\hat w^{(i)}, x_i') - g(\hat w, x_i') \right) \\
    & \le \E_{X, X'} \frac{1}{m} \sum_{i=1}^n L \cdot \left\| w^{(i)} - w \right\| \\
    & \le \frac{4L^2}{\lambda m}.
    \end{aligned}
\end{equation*}

Now for given \(w \in \mathcal{W}\) we have \( G(w) = \E [\hat G(w)] \ge \E [\hat G(\hat w)]\), by definition of \(\hat w\). Therefore, we conclude that
\begin{equation*}
    \E [G(\hat w)] - G(w) \le \frac{4L^2}{\lambda m}.
\end{equation*}
\end{proof}
 
\end{document}